\documentclass[twoside]{article} 

\usepackage[accepted]{aistats2017}

\usepackage[utf8]{inputenc} 
\usepackage{hyperref}       
\usepackage{url}            
\usepackage{booktabs}       
\usepackage{amsfonts}       
\usepackage{nicefrac}       
\usepackage{microtype}      
\usepackage{amsmath,amssymb}
\usepackage{amsthm}
\usepackage{graphicx}
\usepackage[numbers]{natbib}

\newcommand{\bp}{\boldsymbol{p}}
\newcommand{\bq}{\boldsymbol{q}}

\newcommand{\pbar}{\overline{p}}
\newcommand{\qbar}{\overline{q}}

\newcommand{\field}[1]{\mathbb{#1}}

\newcommand{\scX}{\mathcal{X}}

\newcommand{\E}{\field{E}}

\renewcommand{\Pr}{\field{P}}

\newcommand{\theset}[2]{ \left\{ {#1} \,:\, {#2} \right\} }

\newcommand{\Ind}[1]{ \field{I}\left\{{#1}\right\} }

\newcommand{\sgn}{\mbox{\sc sgn}}

\renewcommand{\ss}{\subseteq}
\newcommand{\wh}{\widehat}

\newcommand{\ve}{\varepsilon}

\newcommand{\muhat}{\wh{\mu}}
\newcommand{\Yhat}{\wh{Y}}

\newcommand{\spin}{\{-1,+1\}}

\newcommand{\din}{d_{\mathrm{in}}}

\newcommand{\hdout}{\wh{d}_{\mathrm{out}}}
\newcommand{\hdin}{\wh{d}_{\mathrm{in}}}
\newcommand{\dout}{d_{\mathrm{out}}}

\newcommand{\Nin}{\mathcal{E}_{\mathrm{in}}}
\newcommand{\Nout}{\mathcal{E}_{\mathrm{out}}}
\newcommand{\NNin}{\mathcal{N}_{\mathrm{in}}}
\newcommand{\NNout}{\mathcal{N}_{\mathrm{out}}}
\newcommand{\NN}{\mathcal{N}}

\newcommand{\htr}{\wh{tr}}
\newcommand{\hun}{\wh{un}}

\newcommand{\iin}{i_{\mathrm{in}}}
\newcommand{\iout}{i_{\mathrm{out}}}
\newcommand{\jin}{j_{\mathrm{in}}}

\newcommand{\uin}{U_{\mathrm{in}}}
\newcommand{\uout}{U_{\mathrm{out}}}
\newcommand{\Psiin}{\Psi_{\mathrm{in}}}
\newcommand{\Psiout}{\Psi_{\mathrm{out}}}

\newcommand{\hdeltain}{\wh{\delta}_{\mathrm{in}}}
\newcommand{\hdeltaout}{\wh{\delta}_{\mathrm{out}}}
\newcommand{\tauhat}{\wh{\tau}}
\newcommand{\eps}{\epsilon}

\newtheorem{lemma}{Lemma}
\newtheorem{theorem}{Theorem}

\usepackage[autolanguage,np]{numprint}
\usepackage{caption}
\usepackage{subcaption}
\usepackage{longtable,array,multirow}
\usepackage{xcolor}
\usepackage{csquotes}
\usepackage{enumitem}

\newcolumntype{M}[1]{>{\centering\arraybackslash}m{#1}}
\newcolumntype{N}{@{}m{0pt}@{}}

\usepackage{makecell}

\newcommand{\wik}{\textsc{Wikipedia}}
\newcommand{\sla}{\textsc{Slashdot}}
\newcommand{\epi}{\textsc{Epinion}}
\newcommand{\kiw}{\textsc{Wik. Edits}}
\newcommand{\aut}{\textsc{Citations}}
\newcommand{\ssn}{signed social networks}
\newcommand{\eij}{\ensuremath{i\rightarrow j}}
\newcommand{\eji}{\ensuremath{j\rightarrow i}}
\newcommand{\yij}{y_{i, j}}
\newcommand{\yhat}{\widehat{y}}

\DeclareMathOperator{\rank}{rank}
\newcommand{\diam}[1]{\ensuremath{\mathrm{diam}\left(#1\right)}}
\newcommand{\trainset}{\ensuremath{E_0}}
\newcommand{\vfirst}[1]{\mathbf{\textcolor{brown}{#1}}}
\newcommand{\vsecond}[1]{\mathit{\textcolor{red}{#1}}}
\newcommand{\vfirstSig}[1]{\mathbf{\textcolor{brown}{\underline{#1}}}}
\newcommand{\vsecondSig}[1]{\vsecond{#1}}

\newcommand{\comptriads}{\textsc{16 Triads}}
\newcommand{\complowrank}{\textsc{LowRank}}
\newcommand{\compranknodes}{\textsc{RankNodes}}
\newcommand{\compbayesian}{\textsc{Bayesian}}
\newcommand{\uslogregp}{\textsc{LogReg}}
\newcommand{\usrule}{\textsc{blc}$(tr,un)$}
\newcommand{\uslprop}{\textsc{L. Prop.}}
\newcommand{\uslpropGsec}{\textsc{L. Prop.}}

\newcommand{\qoptim}{\textsc{Unreg.}}

\begin{document}

%

%

\runningtitle{Edge Sign Prediction in Social Networks}

\runningauthor{Géraud Le Falher, Nicol\`o Cesa-Bianchi, Claudio Gentile, Fabio Vitale}
\twocolumn[

\aistatstitle{On the Troll-Trust Model \\ for Edge Sign Prediction in Social Networks}
\vspace{-0.13in}
\aistatsauthor{Géraud Le Falher$\mathbf{^{(1)}}$ \And Nicol\`o
Cesa-Bianchi$\mathbf{^{(2)}}$ \And Claudio Gentile$\mathbf{^{(3)}}$ \And Fabio Vitale$\mathbf{^{(1,4)}}$}
\aistatsaddress{ {\small $\mathbf{^{(1)}}$ Inria, Univ. Lille, CNRS UMR 9189 --
  CRIStAL, France} \hspace{5em} {\small $\mathbf{^{(2)}}$ Universit\`a degli Studi di Milano, Italy} \\
{\small $\mathbf{^{(3)}}$ University of Insubria, Italy} \hspace{5em} {\small $\mathbf{^{(4)}}$ Department of Computer Science, Aalto University, Finland}}
\vspace{-0.08in}
]


\begin{abstract}
\vspace{-0.1in}
  In the problem of edge sign prediction, we are given a directed graph
(representing a social network), and our task is to predict the binary
labels of the edges (i.e., the positive or negative nature of the social
relationships). Many successful heuristics for this problem are based on the
troll-trust features, estimating at each node the fraction of outgoing and
incoming positive/negative edges. We show that these heuristics can be understood,
and rigorously analyzed, as approximators to the Bayes optimal classifier for a
simple probabilistic model of the edge labels. We then show that the maximum
likelihood estimator for this model approximately corresponds to the predictions
of a Label Propagation algorithm run on a transformed version of the original
social graph. Extensive experiments on a number of real-world datasets show that
this algorithm is competitive against state-of-the-art classifiers in terms of
both accuracy and scalability. Finally, we show that troll-trust
features can also be used to derive online learning algorithms which have
theoretical guarantees even when edges are adversarially labeled.

\end{abstract}

\vspace{-0.2in}
\section{Introduction}
\vspace{-0.1in}
%
Connections in social networks are mostly driven by the {\em homophily} assumption: linked individuals tend to be similar, sharing personality traits, attitudes, or interests. However, homophily alone is clearly not sufficient to explain the variety of social links. In fact, sociologists have long studied networks, hereafter called \emph{signed} social networks, where also {\em negative} relationships ---like dissimilarity, disapproval or distrust--- are explicitly displayed. The presence of negative relationships is also a feature of many technology-mediated social networks. Known examples are \textsc{Ebay}, where users trust or distrust agents in the network based on their personal interactions, \textsc{Slashdot}, where each user can tag another user as friend or foe, and \textsc{Epinion}, where users can rate positively or negatively not only products, but also other users. Even in social networks where connections solely represent friendships, negative links can still emerge from the analysis of online debates among users.


When the social network is signed, specific challenges arise in both network analysis and learning. On the one hand, novel methods are required to tackle standard tasks (e.g., user clustering, link prediction, targeted advertising/recommendation, 
analysis of the spreading of diseases in epidemiological models). On the other hand, new problems such as edge sign prediction, which we consider here, naturally emerge. Edge sign prediction is the problem of classifying the positive or negative nature of the links based on the network topology. Prior knowledge of the network topology is often a realistic assumption, for in several situations the discovery of the link sign can be more costly than acquiring the topological information of the network. For instance, when two users of an online social network communicate on a public web page, we immediately detect a link. Yet, the classification of the link sign as positive or negative may require complex techniques. 

From the modeling and algorithmic viewpoints, because of the huge amount of available networked data, a major concern in developing learning methods for edge sign prediction is algorithmic scalability. Many successful, yet simple heuristics for edge sign prediction are based on the troll-trust features, i.e., on the fraction of outgoing negative links (trollness) and incoming positive links (trustworthiness) at each node. We study such heuristics by defining a probabilistic generative model for the signs on the directed links of a given network, and show that these heuristics can be understood and analyzed as approximators to the Bayes optimal classifier for our generative model. We also gather empirical evidence supporting our probabilistic model by observing that a logistic model trained on trollness and trustworthiness features alone is able to learn weights that, on all datasets considered in our experiments, consistently satisfy the properties predicted by our model.

We then introduce suitable graph transformations defining reductions from edge sign prediction to node sign prediction problems. This opens up the possibility of using the arsenal of known algorithmic techniques developed for node classification. In particular, we show that a Label Propagation algorithm, combined with our reduction, approximates the maximum likelihood estimator of our probabilistic generative model. Experiments on real-world data show the competitiveness of our approach in terms of both prediction performance (especially in the regime when training data are scarce) and scalability.

Finally, we point out that the notions of trollness and trustworthiness naturally define a measure of complexity, or learning bias, for the signed network that can also be used to design {\em online} (i.e., sequential) learning algorithms for the edge sign prediction problem. The learning bias encourages settings where the nodes in the network have polarized features (e.g., trollness/trustworthiness are either very high or very low). Our online analysis holds under adversarial conditions, namely, without any stochastic assumption on the assignment of signs to the network links.



\vspace{-0.1in}
\subsection{Related work}
\vspace{-0.1in}
Interest in signed networks can be traced back to the psychological theory of structural balance~\cite{cartwright1956structural,heider1958psychology} with its weak version~\cite{davis1967clustering}. The advent of online signed social networks has enabled a more thorough and quantitative understanding of that phenomenon. Among the several approaches related to our work, some extend the spectral properties of a graph to the signed case in order to find good embeddings for classification~\cite{Kunegis2009,SignedEmbedding15}. However, the use of the adjacency matrix usually requires a quadratic running time in the number of nodes, which makes those methods hardly scalable to large graphs. Another approach is based on mining ego networks with SVM. Although this method seems to deliver good results~\cite{Papaoikonomou2014}, the running time makes it often impractical for large real-world datasets. An alternative approach, based on local features only and proposed in~\cite{Leskovec2010}, relies on the so-called status theory for directed graphs~\cite{guha2004propagation}. Some works in active learning, using a more sophisticated bias based on the correlation clustering (CC) index~\cite{TreeStar12,CCCC12}, provide strong theoretical guarantees. However, the bias used there is rather strong, since it assumes the existence of a $2$-clustering of the nodes with a small CC index.

Whereas our focus will be on {\em binary} prediction, researchers have also considered a weighted version of the problem, where edges measure the amount of trust or distrust between two users~(e.g., \cite{guha2004propagation,tang2013exploiting,
Qian2014sn}). Other works have also considered versions of the problem where side information related to the network is available to the learning system. For instance, \cite{EdgeSignsRating15} uses the product purchased on Epinion in conjunction with a neural network, \cite{TrollDetection15} identifies trolls by analysing the textual content of their post, and~\cite{SNTransfer13} uses SVM to perform transfer learning from one network to another. While many of these approaches have interesting performances, they often require extra information which is not always available (or reliable) and, in addition, may face severe scaling issues.
The recent survey~\cite{tang2015survey} contains pointers to many papers on edge sign prediction for signed networks, especially in the Data Mining area. Additional references, more closely related to our work, will be mentioned at the end of Section~\ref{ss:passive}.

\vspace{-0.1in}
\section{Notation and Preliminaries}\label{s:prel}
\vspace{-0.1in}
In what follows, we let $G=(V,E)$ be a {\em directed} graph, whose edges $(i,j) \in E$ carry a binary label $y_{i,j}\in \spin$. The edge labeling will sometimes be collectively denoted by the $|V|\times |V|$ matrix $Y = [Y_{i,j}]$, where $Y_{i,j} = y_{i,j}$ if $(i,j) \in E$, and $Y_{i,j} = 0$, otherwise. The corresponding edge-labeled graph will be denoted by $G(Y) = (V,E(Y))$. We use $\Nin(i)$ and $\Nout(i)$ to denote, respectively, the set of edges incoming to and outgoing from node $i \in V$, with $\din(i) = \big|\Nin(i)\big|$ and $\dout(i) = \big|\Nout(i)\big|$ being the in-degree and the out-degree of $i$. Moreover, $\din^+(i)$ is the number of edges $(k,i)\in\Nin(i)$ such that $y_{k,i} = +1$. We define $\din^-(i)$, $\dout^+(i)$, and $\dout^-(i)$ similarly, so that, for instance, $\dout^-(i)/\dout(i)$ is the fraction of outgoing edges from node $i$ whose label in $G(Y)$ is $-1$. We call $tr(i) = \dout^-(i)/\dout(i)$ the {\em trollness} of node $i$, and $un(i) = \din^-(i)/\din(i)$ the {\em untrustworthiness} of node $i$. Finally, we also use the notation $\NNin(i)$ and $\NNout(i)$ to represent, respectively, the in-neighborhood and the out-neighborhood of node $i \in V$.

Given the directed graph $G = (V,E)$, we define two {\em edge-to-node reductions} transforming the original graph $G$ into other graphs. As we see later, these reductions are useful in turning the edge sign prediction problem into a {\em node} sign prediction problem (often called node classification problem), for which many algorithms are indeed available
---see, e.g., \cite{BC01,ZGL03,HP07,HLP09,cgvz13}.
Although any node classification method could in principle be used, the reductions we describe next are essentially aimed at preparing the ground for quadratic energy-minimization approaches computed through a {\em Label Propagation} algorithm~(e.g., \cite{ZGL03,BDL06}).

The first reduction, called $G \rightarrow G'$, builds an {\em undirected} graph $G' = (V',E')$ as follows. Each node $i \in V$ has two copies in $V'$, call them $\iin$ and $\iout$. Each directed edge $(i,j)$ in $E$ is associated with one node, call it $e_{i,j}$, in $V'$, along with the two undirected edges $(\iout,e_{i,j})$ and  $(e_{i,j},\jin)$. Hence $|V'| = 2|V|+|E|$ and $|E'| = 2|E|$. Moreover, if $G = G(Y)$ is edge labeled, then this labeling transfers to the subset of nodes $e_{i,j} \in V'$, so that $G'$ is a graph $G'(Y) = (V'(Y),E')$ with partially-labeled nodes. 
The second reduction, called $G \rightarrow G''$, builds an {\em undirected and weighted} graph $G'' = (V'',E'')$.
Specifically, we have $V'' \equiv V'$ and $E'' \supset E'$, where the set $E''$ also includes edges $(\iout,\jin)$ for all $i$ and $j$ such that $(i,j) \in E$.
The edges in $E'$ have weight $2$, whereas the edges in $E''\setminus E'$ have weight $-1$.
Finally, as in the $G \rightarrow G'$ reduction, if $G = G(Y)$ is edge labeled, then this labeling transfers to the subset of nodes $e_{i,j} \in V''$. Graph $G'$, which will not be used in this paper, is an intermediate structure between $G$ and $G''$ and provides a conceptual link to the standard cutsize measure in node sign classification.
Figure~\ref{f:etnr} illustrates the two reductions. 


%

\begin{figure*}[t]
\vspace{-0.12in}
  \centering
  \begin{subfigure}[b]{0.28\textwidth}
    \centering \includegraphics[height=.10\textheight]{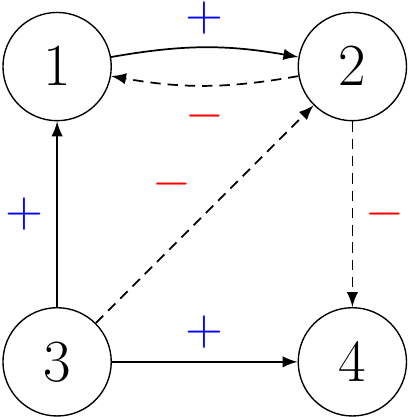} \caption{}
  \end{subfigure}~
  \begin{subfigure}[b]{0.30\textwidth}
    \centering \includegraphics[height=.12\textheight]{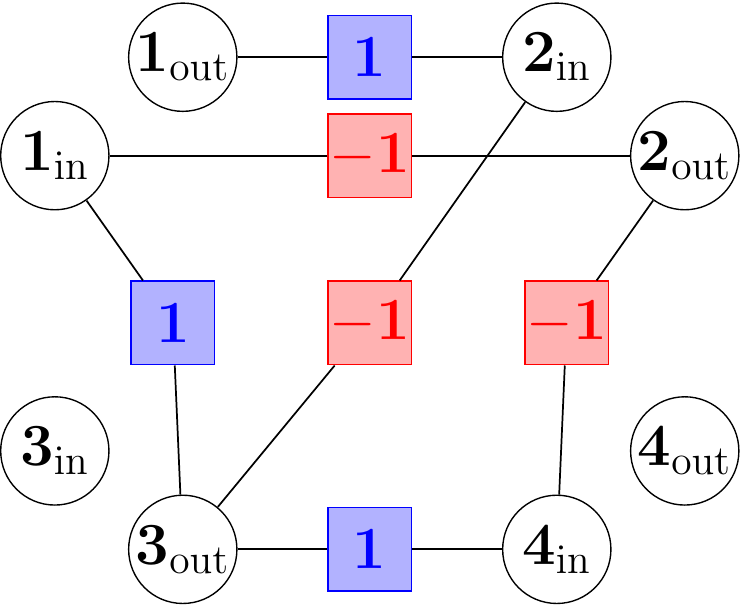} \caption{}
  \end{subfigure}~
  \begin{subfigure}[b]{0.36\textwidth}
    \centering \includegraphics[height=.14\textheight]{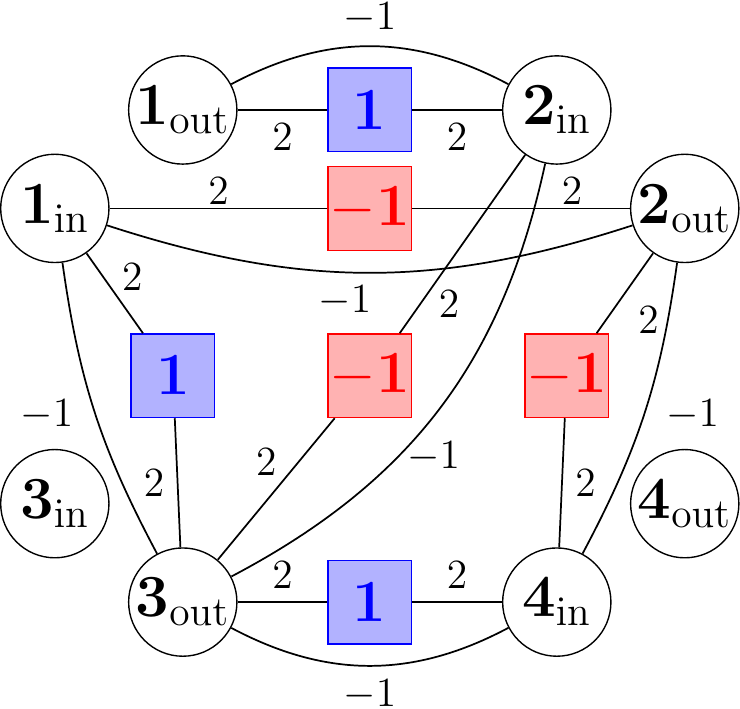} \caption{}
  \end{subfigure}
\vspace{-0.1in}
\caption{\label{f:etnr}
{\bf (a)} A directed edge-labeled graph $G$. {\bf (b)} Its corresponding graph $G'$ resulting from the $G\rightarrow G'$ reduction. The square nodes in $G'$ correspond to the edges in $G$, and carry the same labels as their corresponding edges. On the other hand, the $2|V|$ circle nodes in $G'$ are unlabeled. Observe that some nodes in $G'$ are isolated (and thus unimportant); these are exactly the nodes in $G'$ corresponding to the nodes having in $G$ no outgoing or no incoming edges ---see, e.g., nodes $3$ and $4$ in $G$. {\bf (c)} The weighted graph resulting from the $G\rightarrow G''$ reduction.
}
\vspace{-0.1in}
\end{figure*}

These reductions are meaningful only if they are able to approximately preserve 
label {\em regularity} when moving from edges to nodes. That is, if the edge sign prediction problem is easy for a given $G(Y) = (V,E(Y))$, then the corresponding node sign prediction problems on $G'(Y) = (V'(Y),E')$ and $G''(Y) = (V''(Y),E)$ are also easy, and vice versa.
While we could make this argument more quantitative, here we simply observe that if each node in $G$ tends to be either troll or trustworthy, then few labels from the incoming and outgoing edges of each such node are sufficient to predict the labels on the remaining edges in $G$, and this translates to a small cutsize\footnote
{
Recall that the cutsize of an undirected node-labeled graph $G'(Y)$ is the number of edges in $G'$ connecting nodes having mismatching labels. 
}
of $G'(Y)$ over the nodes corresponding to the edges in $G$ (the colored squares in Figure \ref{f:etnr} (b)).
Again, we would like to point out that these reductions serve two purposes: 
First, they allow us to use the many algorithms designed for the better studied problem of node sign prediction. Second, the reduction $G\rightarrow G''$ with the specific choice of edge weights is designed to make the Label Propagation solution approximate the maximum likelihood estimator associated with our generative model (see Section~\ref{ss:passive}).
Note also that efficient Label Propagation implementations exist that can leverage the sparsity of $G''$.

We consider two learning settings associated with the problem of edge sign prediction: a batch setting and an online setting. In the batch setting, we assume that a training set of edges $E_0$ has been drawn uniformly at random {\em without replacement} from $E$, we observe the labels in $E_0$, and we are interested in predicting the sign of the remaining edges $E \setminus E_0$ by making as few prediction mistakes as possible. 
%
%
The specific batch setting 
we study here assumes that labels are produced by a generative model which we describe in the next section, and our label regularity measure is a quadratic function (denoted by $\Psi^2_{G''}(Y)$ ---see Section~\ref{s:exp} for a definition), related to this model. $\Psi^2_{G''}(Y)$ is small just when all nodes in $G$ tend to be either troll or trustworthy. 

On the other hand, the {\em online} setting we consider is the standard mistake bound model of online learning~\cite{litt88} where all edge labels are assumed to be generated by an adversary and sequentially presented to the learner according to an arbitrary permutation. For an online learning algorithm $A$, we are interested in measuring the total number of mistakes $M_A(Y)$ the algorithm makes over $G(Y)$ when the worst possible presentation order of the edge labels in $Y$ is selected by the adversary.
Also in the online setting our label regularity measure, denoted here by $\Psi_G(Y)$, is small when nodes in $G$ tend to be either troll or trustworthy. 
Formally, for fixed $G$ and $Y$, let 
$
    \Psiin(j,Y) = \min\big\{\din^-(j),\din^+(j)\big\}
$ and
$
    \Psiout(i,Y) = \min\big\{\dout^-(i),\dout^+(i)\big\}
$.
Let also
$
    \Psiin(Y) = \sum_{j \in V} \Psiin(j,Y)
$ and
$
\Psiout(Y) = \sum_{i \in V} \Psiout(i,Y)
$.
Then we define $\Psi_G(Y) = \min\big\{\Psiin(Y),\Psiout(Y)\big\}$. The two measures $\Psi^2_{G''}(Y)$ and $\Psi_G(Y)$ are conceptually related. 
Indeed, their value on real data is quite similar
(see Table~\ref{tab:all_mcc} in Section~\ref{s:exp}).


\vspace{-0.08in}
\section{Generative Model for Edge Labels}\label{s:gen}
\vspace{-0.1in}
We now define the stochastic generative model for edge labels we use in the batch learning setting. Given the graph $G = (V,E)$, let the label $y_{i,j} \in \spin$ of directed edge $(i,j) \in E$ be generated as follows. Each node $i \in V$ is endowed with two latent parameters $p_i, q_i \in [0,1]$, which we assume to be generated, for each node $i$, by an independent draw from a fixed but unknown joint prior distribution $\mu(p,q)$ over $[0,1]^2$. Each label $y_{i,j} \in \spin$ is then generated by an independent draw from the mixture of $p_i$ and $q_j$,
$
	\Pr\big( y_{i,j} = 1 \big) = \tfrac{p_i + q_j}{2}~.
$
The basic intuition is that the nature $y_{i,j}$ of a relationship $i\rightarrow j$ is stochastically determined by a mixture between how much node $i$ tends to like other people ($p_i$) and how much node $j$ tends to be liked by other people ($q_j$). In a certain sense, $1-tr(i)$ is the empirical counterpart to $p_i$, and $1-un(j)$ is the empirical counterpart to $q_j$.\footnote
{
One might view our model as reminiscent of standard models for link generation in social network analysis, like the classical $p_1$ model from \cite{hl81}. Yet, the similarity falls short, for all these models aim at representing the likelihood of the network topology, rather than the probability of edge signs, once the topology is {\em given}.
} 
Notice that the Bayes optimal prediction for $y_{i,j}$ is
$
	y^*(i,j) = \sgn\big(\eta(i,j) - \tfrac{1}{2}\big)~,
$
where $\eta(i,j) = \Pr\big( y_{i,j} = 1 \big)$. Moreover, the probability of drawing at random a $+1$-labeled edge from $\Nout(i)$ and the probability of drawing at random a $+1$-labeled edge from $\Nin(j)$ are respectively equal to
\vspace{-0.12in}

\begin{small}
\begin{equation}\label{e:pout}
\frac{1}{2}\,\Biggl(p_i + \frac{1}{\dout(i)}\!\!\sum_{j\in \NNout(i)} \!\!\! q_j \Biggl) \,\,\,\,\text{and}\,\,\,\,
%
\frac{1}{2}\,\Biggl(q_j + \frac{1}{\din(j)}\!\!\sum_{i\in \NNin(j)} \!\!\! p_i \Biggl)~.
\end{equation}
\end{small}

\vspace{-0.25in}
\section{Algorithms in the Batch Setting}\label{s:algbatch}
\vspace{-0.08in}
Given $G(Y) =(V,E(Y))$, we have at our disposal a training set $E_0$ of labeled edges from $E(Y)$, our goal being that of building a predictive model for the labels of the remaining edges.

Our first algorithm is an approximation to the Bayes optimal predictor $y^*(i,j)$. Let us denote by $\htr(i)$ and $\hun(i)$ the trollness and the untrustworthiness of node $i$ when both are computed on the subgraph induced by the training edges. We now design and analyze an edge classifier of the form
\begin{equation}
\label{eq:predictor}
	\sgn\Big(\big(1-\htr(i)\big) + \big(1-\hun(j)\big) - \tfrac{1}{2} -\tau\Big)~,
\end{equation}
where $\tau \ge 0$ is the only parameter to be trained. Despite its simplicity, this classifier works reasonably well in practice, as demonstrated by our experiments (see Section~\ref{s:exp}). Moreover, unlike previous edge sign prediction methods for directed graphs, our classifier comes with a rigorous theoretical motivation, since it approximates the Bayes optimal classifier $y^*(i,j)$ with respect to the generative model defined in Section~\ref{s:gen}. It is important to point out that when we use $1-\htr(i)$ and $1-\hun(j)$ to estimate $p_i$ and $q_j$, an additive bias shows up due to~(\ref{e:pout}). This motivates the need of a threshold parameter $\tau$ to cancel this bias. Yet, the presence of a prior distribution $\mu(p,q)$ ensures that this bias is the same for all edges $(i,j) \in E$.

Our algorithm works under the assumption that for given parameters $Q$ (a positive integer) and $\alpha \in (0,1)$ there exists a set\footnote
{
$E_L$ is needed to find an estimate $\tauhat$ of $\tau$ in~(\ref{eq:predictor}) ---see Step~3 of the algorithm. Any undirected matching of $G$ of size $\mathcal{O}(\log|V|)$ can be used. In practice, however, we never computed $E_L$, and estimated $\tau$ on the entire training set $E_0$.
}
$E_L \ss E$ of size $\tfrac{2Q}{\alpha}$ where each vertex $i \in V$ appearing as an endpoint of some edge in $E_L$ occurs at most once as origin ---i.e., $(i,j)$--- and at most once as destination ---i.e., $(j,i)$. Moreover, we assume $E_0$ has been drawn from $E$ at random {\em without} replacement, with $m = |E_0| = \alpha\,|E|$. The algorithm performs the following steps:


\begin{enumerate}[leftmargin=*,nosep,label=\textbf{\arabic*.}]
	\item For each $j \in V$, let
		$
		\hun(j) = \hdin^-(j)/\hdin(j)
		$,
		i.e., the fraction of negative edges found in $\Nin(j) \cap E_0$.

	\item For each $i \in V$, let
		$
		\htr(i) = \hdout^-(i)/\hdout(i)
		$,
		i.e., the fraction of negative edges found in $\Nout(i) \cap E_0$.

	\item Let $\tauhat$ be the fraction of positive edges in $E_L\cap E_0$.

	\item Any remaining edge $(i,j) \in E\setminus E_0$ is predicted as
		$
		\yhat(i,j) = \sgn\Big(\big(1-\htr(i)\big) + \big(1-\hun(j)\big) - \tfrac{1}{2} -\tauhat\Big)
		$.
\end{enumerate}

The next result\footnote
{
All proofs are in the supplementary material.
} 
shows that if the graph is not too sparse, then the above algorithm can approximate the Bayes optimal predictor on nodes whose in-degree and out-degree is not too small.
%
%

\vspace{-0.2cm}
\begin{theorem}
\label{t:active}
Let $G(Y) = (V,E(Y))$ be a directed graph with labels on the edges generated according to the model in Section~\ref{s:gen}.
If the algorithm is run with parameter $Q = \Omega(\ln|V|)$, and $\alpha \in (0,1)$ such that the above assumptions are satisfied, then $\yhat(i,j) = y^*(i,j)$ holds with high probability simultaneously for all test edges $(i,j) \in E$ such that $\dout(i),\din(j) = \Omega(\ln|V|)$, and $\eta(i,j) = \Pr(y_{i,j}=1)$ is bounded away from $\tfrac{1}{2}$. 
\end{theorem}
%


The approach leading to Theorem~\ref{t:active} lets us derive the \usrule{}
algorithm assessed in our experiments of Section~\ref{s:exp}, but it needs the graph to be
sufficiently dense and the bias $\tau$ to be the same for all edges. In order to
address these limitations, we now introduce a second method based on label
propagation.



\vspace{-0.15in}
\subsection{Approximation to Maximum Likelihood via Label Propagation}\label{ss:passive}

\vspace{-0.07in}
For simplicity, assume the joint prior distribution $\mu(p,q)$ is uniform over $[0,1]^2$ with independent marginals, and suppose that we draw at random without replacement the training set $E_0 = \big((i_1,j_1),y_{i_1,j_1}), ((i_2,j_2),y_{i_2,j_2}), \ldots, ((i_m,j_m),y_{i_m,j_m}\big)$, with $m = |E_0|$. Then a reasonable approach to approximate $y^*(i,j)$ would be to resort to a maximum likelihood estimator of the parameters $\{p_i, q_i\}_{i=1}^{|V|}$ based on $E_0$.
As showed in the supplementary material,
the gradient of the log-likelihood function w.r.t.\ $\{p_i, q_i\}_{i=1}^{|V|}$ satisfies
\vspace{-0.05in}
\begin{small}
\begin{align}
&\frac{\partial \log \Pr\left(E_0 \,\Big|\, \{p_i, q_i\}_{i=1}^{|V|}\right)}{\partial p_{\ell}}\label{e:mlp}\\
&\ \ = 
\sum_{k=1}^m
\frac{\Ind{i_k = \ell,y_{\ell,j_k}=+1}}{p_{\ell}+q_{j_k}}\, 
- \sum_{k=1}^m
\frac{\Ind{i_k = \ell,y_{\ell,j_k}=-1}}{2-p_{\ell}-q_{j_k}}\,,\notag
\end{align}
\vspace{-0.17in}
\begin{align}
&\frac{\partial \log \Pr\left(E_0 \,\Big|\, \{p_i, q_i\}_{i=1}^{|V|}\right)}{\partial q_{\ell}}\label{e:mlq}\\
&\ \ = 
\sum_{k=1}^m
\frac{\Ind{j_k = \ell,y_{i_k,\ell}=+1}}{p_{i_k}+q_{\ell}}\, 
- \sum_{k=1}^m
\frac{\Ind{j_k = \ell,y_{i_k,\ell}=-1}}{2-p_{i_k}-q_{\ell}}\,,\notag
\end{align}
\end{small}
where $\Ind{\cdot}$ is the indicator function of the event at argument.
Unfortunately, equating~(\ref{e:mlp}) and~(\ref{e:mlq}) to zero, and solving for parameters $\{p_i, q_i\}_{i=1}^{|V|}$ gives rise to a hard set of nonlinear equations.
Moreover, some such parameters may never occur in these equations, namely whenever $\Nout(i)$ or $\Nin(j)$ are not represented in $E_0$ for some $i,j\in V$. 
%
Our \emph{first approximation} is therefore to replace the nonlinear equations resulting from~(\ref{e:mlp}) and~(\ref{e:mlq}) by the following set of linear equations\footnote{Details are provided in the supplementary material.}, one for each $\ell \in V$:
\vspace{-0.08in}
\begin{align*}
\sum_{k=1}^m &\Ind{i_k = \ell,y_{\ell,j_k}=+1} \left(2-p_{\ell}-q_{j_k}\right)\\
=
&\sum_{k=1}^m \Ind{i_k = \ell,y_{\ell,j_k}=-1}
(p_{\ell}+q_{j_k})
\end{align*}
%
%
\vspace{-0.2in}
\begin{align*}
\sum_{k=1}^m &\Ind{j_k = \ell,y_{i_k,\ell}=+1} \left(2-p_{i_k}-q_{\ell}\right)\\
=
&\sum_{k=1}^m  \Ind{j_k = \ell,y_{i_k,\ell}=-1}
\left(p_{i_k}+q_{\ell}\right)~.
\end{align*}
%
%
%
The solution to these equations are precisely the points where the gradient w.r.t.\ $(\bp,\bq) =\{p_i, q_i\}_{i=1}^{|V|}$ of the quadratic function
\vspace{-0.1in}
\[
f_{E_0}(\bp,\bq) = \sum_{(i,j) \in E_0} \left(\frac{1+y_{i,j}}{2} - \frac{p_i+q_j}{2} \right)^2
\]
vanishes.
We follow a label propagation approach by adding to $f_{E_0}$ the corresponding test set function $f_{E\setminus E_0}$, and treat the sum of the two as the function to be minimized during training w.r.t.\ both $(\bp,\bq)$ and all $y_{i,j} \in [-1,+1]$ for $(i,j)\in E\setminus E_0$, i.e.,
\begin{equation}\label{e:quadratic}
\min_{(\bp,\bq), y_{i,j} \in [-1,+1],\,(i,j)\in E\setminus E_0} 
\left(f_{E_0}(\bp,\bq) + f_{E\setminus E_0}(\bp,\bq)\right)\,.
\end{equation}
Binary $\pm 1$ predictions on the test set $E\setminus E_0$ are then obtained by thresholding the obtained values $y_{i,j}$ at $0$.

We now proceed to solve \eqref{e:quadratic} via label propagation~\cite{ZGL03} on the graph $G''$ obtained through the $G \rightarrow G''$ reduction of Section~\ref{s:prel}.
However, because of the presence of negative edge weights in $G''$, we first have to symmetrize\footnote{%
While we note here that such linear transformation of the variables does not change the problem, we provide more details in Section~1.3 of the supplementary material.} variables $p_i, q_i, y_{i,j}$ so as they all lie in the interval $[-1,+1]$.
After this step, one can 
see that, once we get back to the original variables, label propagation computes the harmonic solution minimizing the function
%
\begin{align*}
{\widehat f}&\big(\bp,\bq,{y_{i,j}}_{(i,j) \in E\setminus E_0}\big)
= f_{E_0}(\bp,\bq) + f_{E\setminus E_0}(\bp,\bq)\\ 
& \qquad + \frac{1}{2}\sum_{i\in V} \left(\dout(i)\Bigl(p_i-\frac{1}{2}\Bigl)^2+\din(i)\Bigl(q_i-\frac{1}{2}\Bigl)^2\right)\,.
\end{align*}
%
The function ${\widehat f}$ is thus a regularized version of the target function $f_{E_0} + f_{E\setminus E_0}$ in (\ref{e:quadratic}), where the regularization term 
tries to enforce the extra constraint that whenever a node $i$ has a high out-degree then the corresponding $p_i$ should be close to $1/2$. Thus, on any edge $(i,j)$ departing from $i$, the Bayes optimal predictor $y^*(i,j) = \sgn(p_i+q_j-1)$ will mainly depend on $q_j$ being larger or smaller than $\tfrac{1}{2}$ (assuming $j$ has small in-degree). Similarly, if $i$ has a high in-degree, then the corresponding $q_i$ should be close to $1/2$ implying that on any edge $(j,i)$ arriving at $i$ the Bayes optimal predictor $y^*(j,i)$ will mainly depend on $p_j$ (assuming $j$ has small out-degree). Put differently, a node having a huge out-neighborhood makes each outgoing edge ``count less" than a node having only a small number of outgoing edges, and similarly for in-neighborhoods.
%
%
%
%
The label propagation algorithm operating on $G''$ does so (see again Figure \ref{f:etnr} (c)) by iteratively updating as follows:
\vspace{-0.04in}
\begin{align*}
	p_{i}  & \leftarrow \frac{-\sum_{j \in \NNout(i)} q_{j} + \sum_{j \in \NNout(i)} (1+y_{i,j})}{3\,\dout(i)}\,\quad\forall i\in V\\ 
	q_{j}  & \leftarrow \frac{-\sum_{i \in  \NNin(j)} p_{i} + \sum_{i \in \NNin(j) } (1+y_{i,j})}{3\,\din(j)}\qquad\forall j \in V\\
  y_{i,j} & \leftarrow \frac{p_i + q_j}{2}~ \quad \forall (i,j) \in E\setminus E_0~.
\end{align*}
%
The algorithm is guaranteed to converge~\cite{ZGL03} to the minimizer of ${\widehat f}$. Notice that the presence of negative weights on the edges of $G''$ does not prevent label propagation from converging. 
This is the algorithm we will be championing in our experiments of Section~\ref{s:exp}.

%


\label{ssec:further_related}
{\bf Further related work.} The vast majority of existing edge sign prediction algorithms for directed graphs are based on the computation of local features of the graph. These features are evaluated on the subgraph induced by the training edges, and the resulting values are used to train a supervised classification algorithm (e.g., logistic regression). The most basic set of local features used to classify a given edge $(i,j)$ are defined by 
$\din^+(j),\din^-(j),\dout^+(i),\dout^-(i)$ computed over the training set $E_0$, and by the embeddedness coefficient $\big|\Nout(i) \cap \Nin(j)\big|$. In turn, these can be used to define more complicated features, such as
\(
	\frac{\din^+(j) + |E^+|\uin(j)}{\din(j) + \uin(j)}
\quad\text{and}\quad
	\frac{\dout^+(i) + p^+\uout(i)}{\dout(i) + \uout(i)}
\)
introduced in~\cite{Bayesian15}, together with their negative counterparts, where $|E^+|$ is the overall fraction of positive edges, and $\uin(j),\uout(i)$ are, respectively, the number of test edges outgoing from $i$ and the number of test edges incoming to $j$.
Other types of features are derived from social status theory~(e.g., \cite{Leskovec2010}), and involve the so-called triads; namely, the triangles formed by $(i,j)$ together with $(i,w)$ and $(w,j)$ for any $w \in \NNout(i) \cap \NNin(j)$. A third group of features is based on node ranking scores. These scores are computed using a variery of methods, including Prestige~\cite{zolfaghar2010mining}, exponential ranking~\cite{traag2010exponential}, PageTrust~\cite{de2008pagetrust}, Bias and Deserve~\cite{mishra2011finding}, TrollTrust~\cite{wu2016troll}, and generalizations of PageRank and HITS to signed networks~\cite{shahriari2014ranking}. Examples of features using such scores are \textsl{reputation} and \textsl{optimism}~\cite{shahriari2014ranking}, defined for a node $i$ by
\(
	\frac{\sum_{j \in \NNin(i)} y_{j,i}\sigma(j)}{\sum_{j \in \NNin(i)} \sigma(j)}
\quad\text{and}\quad
	\frac{\sum_{j \in \NNout(i)} Y_{i,j}\sigma(j)}{\sum_{j \in \NNout(i)} \sigma(j)}\,,
\)
where $\sigma(j)$ is the ranking score assigned to node $j$. Some of these algorithms will be used as representative competitors in our experimental study of Section~\ref{s:exp}.

\vspace{-0.09in}
\section{Algorithms in the Online Setting}\label{s:algonline}

\vspace{-0.12in}
For the online scenario, we have the following result.
%
\begin{theorem}\label{t:online}
There exists a randomized online prediction algorithm $A$ whose expected number of mistakes satisfies
$
\E M_A(Y) = \Psi_G(Y) + O\left(\sqrt{|V|\Psi_G(Y)} + |V| \right)
$
on any edge-labeled graph $G(Y) = (V,E(Y))$.
\end{theorem}
%
The algorithm used in Theorem~\ref{t:online} is a combination of randomized Weighted Majority instances. Details are reported in the supplementary material. We complement the above result by providing a mistake lower bound. Like Theorem~\ref{t:online}, the following result holds for all graphs, and for all label irregularity levels $\Psi_G(Y)$.
%
\begin{theorem}
\label{t:mistake_bound}
Given any edge-labeled graph $G(Y) = (V,E(Y))$ and any integer $K \le \big\lfloor \tfrac{|E|}{2}\big\rfloor$, a randomized labeling $Y\in\spin^{|E|}$ exists such that $\Psi_G(Y) \leq K$, and the expected number of mistakes that any online algorithm $A$ can be forced to make satisfies
\(
\E M_A(Y) \ge \frac{K}{2}\,.
\)
Moreover, as $\frac{K}{|E|} \rightarrow 0$ then
\(
\E M_A(Y) = K
\).
\end{theorem}
%


\vspace{-0.1in}
\section{Experimental Analysis}\label{s:exp}
\vspace{-0.1in}
We now evaluate our edge sign classification methods on representative real-world datasets of varying density and label regularity,
showing that our methods compete well against existing approaches in terms of both predictive and computational performance. We are especially interested in small training set regimes, and have restricted our comparison to the batch learning scenario since all competing methods we are aware of have been developed in that setting only.

{\bf Datasets.} We considered five real-world classification datasets. The first three are directed \ssn{} widely used as benchmarks for this task~(e.g.,\cite{Leskovec2010,shahriari2014ranking,wu2016troll}): In \wik{}, there is an edge from user $i$ to user $j$ if $j$ applies for an admin position and $i$ votes for or against that promotion. In \sla{}, a news sharing and commenting website, member $i$ can tag other members $j$ as friends or foes. Finally, in \epi{}, an online shopping website, user $j$ reviews products and, based on these reviews, another user $i$ can display whether he considers $j$ to be reliable or not. In addition to these three datasets, we considered two other \ssn{} where the signs are inferred automatically.
In \kiw{}~\cite{wikiedits11}, an edge from Wikipedia user $i$ to user $j$ indicates whether they edited the same article in a constructive manner or not.\footnote
{
This is the \href{http://konect.uni-koblenz.de/networks/wikisigned-k2}{KONECT version of
the \enquote{Wikisigned} dataset}, from which we removed self-loops.
}
Finally, in the \aut{}~\cite{kumar2016structure} network, an author $i$ cites another author $j$ by either endorsing or criticizing $j$'s work. The edge sign is derived by classifying the citation sentiment with a simple, yet powerful, keyword-based technique using a list of positive and negative words. See \cite{kumar2016structure} for more details.\footnote{
We again removed self-loops and merged multi-edges which are all of the same sign.
}

\begin{table*}[bt]
  \centering
  \small
  \caption{Dataset properties. The 5th column gives the fraction of positive labels. The last two columns provide two different measures of label regularity ---see main text.\label{tab:dataset}}
\vspace{-0.07in}
  \begin{tabular}{lrrrrrrrr}
    \toprule
    Dataset &       $|V|$ &       $|E|$ &$\frac{|E|}{|V|}$ & $\frac{|E^+|}{|E|}$ & $\frac{\Psi^2_{G''}(Y)}{|E|}$ & $\frac{\Psi_G(Y)}{|E|}$ \\ 
    \midrule
    \aut{} &   \np{4831} &  \np{39452} & 8.1  &              72.33\% &                    .076 &                       .191 \\ 
    \wik{} &   \np{7114} & \np{103108} & 14.5 &              78.79\% &                    .063 &                       .142 \\ 
    \sla{} &  \np{82140} & \np{549202} & 6.7  &              77.40\% &                    .059 &                       .143 \\ 
    \kiw{} & \np{138587} & \np{740106} & 5.3  &              87.89\% &                    .034 &                       .086 \\ 
    \epi{} & \np{131580} & \np{840799} & 6.4  &              85.29\% &                    .031 &                       .074 \\ 
    \bottomrule
  \end{tabular}
\vspace{-0.15in}
\end{table*}

\autoref{tab:dataset} summarizes statistics for these datasets. We note that most edge labels are positive. Hence, test set accuracy is not an appropriate measure of prediction performance. We instead evaluated our performance using the so-called Matthews Correlation Coefficient (MCC)~(e.g., \cite{mcc00}), defined as
\vspace{-0.02in}
\[
	\mathrm{MCC} = \frac{tp\times tn-fp\times fn} {\sqrt{ (tp + fp) ( tp + fn ) ( tn + fp ) ( tn + fn ) } }\,.
\]
\vspace{-0.02in}
MCC combines all the four quantities found in a binary confusion matrix ($t$rue $p$ositive, $t$rue $n$egative, $f$alse $p$ositive and $f$alse $n$egative) into a single metric which ranges from $-1$ (when all predictions are incorrect) to $+1$ (when all predictions are correct).

Although the semantics of the edge signs is not the same across these networks, we can see from \autoref{tab:dataset} that our generative model essentially fits all of them. Specifically, the last two columns of the table report the rate of label (ir)regularity, as measured by $\Psi^2_{G''}(Y)/|E|$ (second-last column) and $\Psi_{G}(Y)/|E|$ (last column), where 
\vspace{-0.1in}
\[
\Psi^2_{G''}(Y) = \min_{(\bp,\bq)} \left(f_{E_0}(\bp,\bq) + f_{E\setminus E_0}(\bp,\bq)\right)\,,
\]
%
$f_{E_0}$ and $f_{E\setminus E_0}$ being the quadratic criterions of Section~\ref{ss:passive}, viewed as functions of both $(\bp,\bq)$, and $y_{i,j}$,
and $\Psi_{G}(Y)$ is the label regularity measure adopted in the online setting, as defined in Section~\ref{s:prel}. 
It is reasonable to expect that higher label irregularity corresponds to lower prediction performance. This trend is in fact confirmed by our experimental findings: whereas \epi{} tends to be easy, \aut{} tends to be hard, and this holds for all algorithms we tested, even if they do not explicitly comply with our inductive bias principles. Moreover, $\Psi^2_{G''}(Y)/|E|$ tends to be proportional to $\Psi_{G}(Y)/|E|$ across datasets, hence confirming the anticipated connection between the two regularity measures.

\vspace{-0.05in}
{\bf Algorithms and parameter tuning.} 
We compared the following algorithms:
%
%

\vspace{-0.07in}
{\bf 1.} The label propagation algorithm of Section~\ref{ss:passive} (referred to as \uslpropGsec{}). The actual binarizing threshold was set by cross-validation on the training set.
%

\vspace{-0.07in}
{\bf 2.} The algorithm analyzed at the beginning of Section~\ref{s:algbatch}, which we call \usrule{} (Bayes Learning Classifier based on {\em tr}ollness and {\em un}trustworthiness). After computing $\htr(i)$ and $\hun(i)$ on training set $E_0$ for all $i \in V$ (or setting those values to $\frac{1}{2}$ in case there is no outgoing or incoming edges for some node), we use Eq.~\eqref{eq:predictor} and estimate $\tau$ on $E_0$.
%

\vspace{-0.07in}
{\bf 3.} A logistic regression model where each edge $(i,j)$ is associated with the features $[1-\htr(i), 1-\hun(j)]$ computed again on $E_0$ (we call this method \uslogregp{}). Best binary thresholding is again computed on $E_0$. Experimenting with this logistic model serves to support the claim we made in the introduction that our generative model in Section~\ref{s:gen} is a good fit for the data.

\vspace{-0.05in}
{\bf 4.} The solution obtained by directly solving the unregularized problem (\ref{e:quadratic}) through a fast constrained minimization algorithm (referred to as \qoptim{}). Again, the actual binarizing threshold was set by cross-validation on the training set.\footnote
{
We have also tried to minimize (\ref{e:quadratic}) by removing the $[-1,+1]$ constraints, but got similar MCC results
as the ones we report for \qoptim{}
}

\vspace{-0.07in}
{\bf 5.} The matrix completion method from~\cite{LowRankCompletion14} based on \complowrank{} matrix factorization. Since the authors showed their method to be robust to the choice of the rank parameter $k$, we picked $k=7$ in our experiments.
%

\vspace{-0.07in}
{\bf 6.} A logistic regression model built on \comptriads{} features derived from status theory~\cite{Leskovec2010}.
%

\vspace{-0.07in}
{\bf 7.} The PageRank-inspired algorithm from \cite{wu2016troll}, where a recursive notion of trollness is computed by solving a suitable set of nonlinear equations through an iterative method, and then used to assign ranking scores to nodes, from which (un)trustworthiness features are finally extracted for each edge. We call this method \compranknodes{}. As for hyperparameter tuning ($\beta$ and $\lambda_1$ in~\cite{wu2016troll}), we closely followed the authors' suggestion of doing cross validation.
%

\vspace{-0.07in}
{\bf 8.} The last competitor is the logistic regression model whose features have been build according to \cite{Bayesian15}. We call this method \compbayesian{}.

The above methods can be roughly divided into {\em local} and {\em global} methods. A local method hinges on building local predictive features, based on neighborhoods: \usrule{}, \uslogregp{}, \comptriads{}, and \compbayesian{} essentially fall into this category. The remaining methods 
are global in that their features are designed to depend on global properties of the graph topology.

\begin{table*}[t]
  \centering
\setlength{\tabcolsep}{3pt}
\scriptsize
\caption{MCC with increasing training set size, with one standard deviation over 12 random sampling of $\trainset$. The last four columns refer to the methods we took from the literature. For the sake of readability, we multiplied all MCC values by 100. The best number in each row is highlighted in \textbf{\textcolor{brown}{bold brown}} and the second one in \textit{\textcolor{red}{italic red}}. If the difference is statistically significant ($p$-value of a paired Student's $t$-test less than $0.005$), the best score is underlined. The ``time" rows contain the time taken to train on a $15\%$ training set.\label{tab:all_mcc}}
\vspace{-0.1in}
\begin{tabular}{lrcccc|cccc}
\toprule
                                                  & $\frac{|\trainset{}|}{|E|}$ &                 \uslpropGsec{} &                       \usrule{} &                    \uslogregp{} &         \qoptim{} &     \complowrank{} &      \comptriads{} &   \compranknodes{} &              \compbayesian{} \\
\midrule
\multirow{7}{*}{\rotatebox[origin=c]{90}{\aut{}}} & $5\%$                       & $\vfirstSig{24.54} \pm 0.69$  & $\vsecondSig{20.21} \pm 0.66$ & $20.19 \pm 0.71$              & $15.86 \pm 0.81$ & $12.76 \pm 0.65$ & $11.04 \pm 0.81$ & $17.18 \pm 1.11$             & $15.28 \pm 1.31$ \\
                                                  & $10\%$                      & $\vfirstSig{31.20} \pm 0.58$  & $\vsecondSig{27.54} \pm 0.56$ & $27.49 \pm 0.62$              & $25.36 \pm 0.78$ & $17.81 \pm 0.76$ & $16.99 \pm 0.63$ & $25.36 \pm 0.85$             & $24.74 \pm 0.59$ \\
                                                  & $15\%$ &  $\vfirstSig{35.66} \pm 0.68$  &  $\vsecondSig{32.87} \pm 0.58$  &               $32.79 \pm 0.60$  &  $31.39 \pm 0.75$  &  $22.58 \pm 0.53$  &  $21.55 \pm 0.91$  &  $30.60 \pm 0.87$  &            $31.71 \pm 0.99$  \\
                                                  & $20\%$ &  $\vfirstSig{38.67} \pm 0.48$  &  $\vsecondSig{36.94} \pm 0.51$  &               $36.86 \pm 0.48$  &  $35.47 \pm 0.41$  &  $25.80 \pm 0.94$  &  $24.27 \pm 0.56$  &  $35.01 \pm 0.83$  &            $36.13 \pm 0.75$  \\
                                                  & $25\%$ &     $\vfirst{41.05} \pm 0.73$  &               $39.83 \pm 0.58$  &               $39.76 \pm 0.59$  &  $38.48 \pm 0.55$  &  $29.67 \pm 0.78$  &  $26.85 \pm 0.87$  &  $38.06 \pm 0.86$  &  $\vsecond{40.34} \pm 0.94$  \\
                                                  & time   &                           19.6 &                             0.6 &                             2.6 &               2835 &               3279 &                6.2 &                155 &                         4813 \\
\midrule
\multirow{7}{*}{\rotatebox[origin=c]{90}{\wik{}}} & $5\%$                       & $\vfirstSig{39.46} \pm 0.79$  & $38.03 \pm 0.97$              & $\vsecondSig{38.50} \pm 0.87$ & $35.72 \pm 0.70$ & $24.58 \pm 1.18$ & $9.59 \pm 1.10$  & $33.60 \pm 0.64$             & $26.45 \pm 0.57$ \\
                                                  & $10\%$                      & $\vsecond{47.17} \pm 0.35$    & $46.03 \pm 0.49$              & $\vfirst{47.22} \pm 0.40$     & $44.53 \pm 0.48$ & $31.72 \pm 0.61$ & $26.36 \pm 0.83$ & $43.21 \pm 0.81$             & $40.28 \pm 0.69$ \\
                                                  & $15\%$ &     $\vsecond{50.49} \pm 0.33$  &  $49.89 \pm 0.40$  &      $\vfirst{50.87} \pm 0.36$  &  $49.08 \pm 0.33$  &  $35.77 \pm 0.58$  &  $33.64 \pm 0.83$  &  $48.50 \pm 0.47$  &  $47.07 \pm 0.38$  \\
                                                  & $20\%$ &  $\vsecondSig{52.74} \pm 0.31$  &  $52.24 \pm 0.49$  &   $\vfirstSig{53.13} \pm 0.27$  &  $51.79 \pm 0.35$  &  $37.90 \pm 0.27$  &  $38.41 \pm 0.53$  &  $51.49 \pm 0.43$  &  $50.54 \pm 0.39$  \\
                                                  & $25\%$ &     $\vsecond{54.00} \pm 0.63$  &  $53.42 \pm 0.59$  &      $\vfirst{54.26} \pm 0.37$  &  $53.31 \pm 0.37$  &  $40.16 \pm 0.57$  &  $41.34 \pm 1.07$  &  $53.30 \pm 0.37$  &  $52.92 \pm 0.48$  \\
                                                  & time   &                            41.9 &                1.6 &                             6.0 &              10629 &               8523 &               14.8 &                249 &              12507 \\
\midrule
\multirow{7}{*}{\rotatebox[origin=c]{90}{\sla{}}} & $5\%$                       & $\vsecondSig{40.77} \pm 0.20$ & $36.13 \pm 0.57$              & $37.00 \pm 0.29$              & $33.49 \pm 0.32$ & $36.83 \pm 0.47$ & $27.10 \pm 0.75$ & $\vfirstSig{45.16} \pm 0.59$ & $29.25 \pm 0.23$ \\
                                                  & $10\%$                      & $\vsecondSig{46.61} \pm 0.29$ & $41.89 \pm 0.39$              & $43.15 \pm 0.21$              & $40.92 \pm 0.23$ & $39.57 \pm 0.27$ & $40.38 \pm 1.47$ & $\vfirstSig{47.84} \pm 0.50$ & $38.25 \pm 0.21$ \\
                                                  & $15\%$ &   $\vfirstSig{49.62} \pm 0.22$  &  $45.42 \pm 0.36$  &  $46.42 \pm 0.16$  &  $45.56 \pm 0.19$  &  $41.21 \pm 0.19$  &  $45.88 \pm 1.01$  &  $\vsecondSig{48.75} \pm 0.71$  &  $43.47 \pm 0.16$  \\
                                                  & $20\%$ &     $\vsecond{51.88} \pm 0.24$  &  $47.78 \pm 0.25$  &  $48.66 \pm 0.10$  &  $48.10 \pm 0.30$  &  $42.74 \pm 0.44$  &  $48.79 \pm 0.57$  &      $\vfirst{52.10} \pm 0.33$  &  $46.89 \pm 0.27$  \\
                                                  & $25\%$ &     $\vsecond{53.12} \pm 0.20$  &  $49.39 \pm 0.24$  &  $50.22 \pm 0.12$  &  $50.11 \pm 0.20$  &  $44.24 \pm 0.44$  &  $50.62 \pm 0.53$  &      $\vfirst{53.29} \pm 0.22$  &  $49.42 \pm 0.22$  \\
                                                  & time   &                             677 &                8.3 &               32.8 &              78537 &              69988 &                131 &                            2441 &              68085 \\
\midrule
\multirow{7}{*}{\rotatebox[origin=c]{90}{\epi{}}} & $5\%$                       & $\vsecondSig{54.83} \pm 0.16$ & $46.94 \pm 0.80$              & $49.16 \pm 0.32$              & $42.79 \pm 0.34$ & $39.96 \pm 0.60$ & $42.94 \pm 2.06$ & $\vfirstSig{56.04} \pm 0.76$ & $37.99 \pm 0.49$ \\
                                                  & $10\%$                      & $\vsecondSig{58.94} \pm 0.27$ & $54.03 \pm 0.46$              & $55.90 \pm 0.13$              & $53.43 \pm 0.39$ & $44.50 \pm 0.52$ & $50.29 \pm 1.07$ & $\vfirstSig{60.60} \pm 0.32$ & $49.90 \pm 0.36$ \\
                                                  & $15\%$ &  $\vsecondSig{61.47} \pm 0.21$  &  $57.63 \pm 0.45$  &  $59.25 \pm 0.17$  &  $58.80 \pm 0.32$  &  $48.24 \pm 0.58$  &  $54.64 \pm 1.62$  &  $\vfirstSig{62.69} \pm 0.21$  &            $56.94 \pm 0.65$  \\
                                                  & $20\%$ &  $\vsecondSig{63.17} \pm 0.13$  &  $60.15 \pm 0.40$  &  $61.45 \pm 0.17$  &  $61.86 \pm 0.13$  &  $52.21 \pm 0.37$  &  $57.27 \pm 1.42$  &  $\vfirstSig{64.10} \pm 0.12$  &            $61.18 \pm 0.45$  \\
                                                  & $25\%$ &               $64.05 \pm 0.20$  &  $61.88 \pm 0.38$  &  $62.89 \pm 0.12$  &  $63.42 \pm 0.14$  &  $54.68 \pm 0.62$  &  $58.42 \pm 1.59$  &     $\vfirst{65.40} \pm 0.85$  &  $\vsecond{64.59} \pm 0.30$  \\
                                                  & time   &                            1329 &               10.1 &               54.0 &             143881 &             127654 &                209 &                           3174 &                       104305 \\
\midrule
\multirow{7}{*}{\rotatebox[origin=c]{90}{\kiw{}}} & $5\%$                       & $\vfirstSig{36.36} \pm 0.53$  & $\vsecondSig{30.89} \pm 0.28$ & $30.81 \pm 0.20$              & $21.69 \pm 0.25$ & $23.15 \pm 0.26$ & $3.04 \pm 0.46$  & $26.63 \pm 0.44$             & $26.68 \pm 0.34$ \\
                                                  & $10\%$                      & $\vfirstSig{38.58} \pm 0.74$  & $35.68 \pm 0.22$              & $\vsecondSig{35.93} \pm 0.16$ & $29.75 \pm 0.21$ & $27.07 \pm 0.44$ & $12.34 \pm 0.79$ & $33.85 \pm 0.33$             & $35.00 \pm 0.34$ \\
                                                  & $15\%$ &  $\vsecondSig{39.08} \pm 0.55$  &               $37.77 \pm 0.22$  &               $38.27 \pm 0.19$  &  $33.61 \pm 0.11$  &  $30.05 \pm 0.29$  &  $17.95 \pm 0.92$  &  $36.88 \pm 0.32$  &  $\vfirstSig{40.00} \pm 0.26$  \\
                                                  & $20\%$ &               $39.04 \pm 0.69$  &               $38.88 \pm 0.36$  &  $\vsecondSig{39.55} \pm 0.11$  &  $35.04 \pm 0.17$  &  $32.17 \pm 0.31$  &  $21.44 \pm 0.67$  &  $38.60 \pm 0.31$  &  $\vfirstSig{43.32} \pm 0.22$  \\
                                                  & $25\%$ &               $38.90 \pm 0.45$  &               $39.41 \pm 0.16$  &  $\vsecondSig{40.44} \pm 0.14$  &  $36.18 \pm 0.20$  &  $33.94 \pm 0.74$  &  $23.41 \pm 0.41$  &  $39.75 \pm 0.32$  &  $\vfirstSig{45.76} \pm 0.29$  \\
                                                  & time   &                             927 &                             9.6 &                            46.8 &             219109 &             129460 &                177 &               3890 &                          92719 \\
\bottomrule
\end{tabular}
\vspace{-0.1in}
\end{table*}

{\bf Results.} Our main results are summarized in Table~\ref{tab:all_mcc}, reporting MCC test set performance after training on sets of varying size (from 5\% to 25\%). Results have been averaged over 12 repetitions. Because scalability is a major concern 
on sizeable datasets, we also give an idea of relative training times (in milliseconds) by reporting the time it took to train a single run of each algorithm on a training set of size\footnote
{
Comparison of training time performances is fair since all algorithms have been carefully implemented using the same stack of Python libraries, and run on the same machine (16 Xeon cores and 192Gb Ram).
} 
15\% of $|E|$, and then predict on the test set. Though our experiments are not conclusive, some trends can be 
spotted:
%
%

\vspace{-0.05in}
{\bf 1.}
Global methods tend to outperform local methods in terms of prediction performance, but are also significantly (or even much) slower (running times can differ by as much as three orders of magnitude). This is not surprising, and is in line with previous experimental findings (e.g., \cite{shahriari2014ranking,wu2016troll}). \compbayesian{} looks like an exception to this rule, but its running time is indeed in the same ballpark as global methods.

\vspace{-0.05in}
{\bf 2.} 
\uslpropGsec{} always ranks first or at least second in this comparison when MCC is considered. On top of it, \uslpropGsec{} is fastest among the global methods (one or even two orders of magnitude faster), thereby showing the benefit of our approach to edge sign prediction.

\vspace{-0.05in}
{\bf 3.} The regularized solution computed by \uslpropGsec{} is always better than the unregularized one computed by \qoptim{} in terms of both MCC and running time.

\vspace{-0.05in}
{\bf 4.} 
As claimed in the introduction, our Bayes approximator \usrule{} closely mirrors in performance the more involved \uslogregp{} model. In fact, supporting our generative model of Section~\ref{s:gen}, the logistic regression weights for features $1-\htr(i)$ and $1-\hun(j)$ are almost equal (see Table~2 in the supplementary material), thereby suggesting that predictor~\eqref{eq:predictor}, derived from the theoretical results at the beginning of Section~\ref{s:algbatch}, is {\em also} the best logistic model based on trollness and untrustworthiness.

\vspace{-0.13in}
\section{Conclusions and Ongoing Research}
\vspace{-0.15in}
%
We have studied the edge sign prediction problem in directed graphs in both batch and online learning settings. In both cases, the underlying modeling assumption hinges on the trollness and (un)trustworthiness predictive features. We have introduced a simple generative model for the edge labels 
to craft this problem as a node sign prediction problem to be efficiently tackled by standard Label Propagation algorithms. 
Furthermore, we have studied the problem in an (adversarial) online setting providing 
upper and (almost matching) lower bounds on the expected number of prediction mistakes.


\vspace{-0.07in}
Finally, we validated our theoretical results by experimentally assessing our methods on five real-world datasets in the small training set regime. 
Two interesting conclusions from our experiments are: i. Our generative model is robust, for it produces Bayes optimal predictors which tend to be empirically best also within the larger set of models that includes all logistic regressors based on trollness and trustworthiness alone; ii. 
our methods are in practice either strictly better than their competitors in terms of prediction quality or, when they are not, they are faster.
We are currently engaged in 
extending our approach so as to incorporate further predictive features (e.g., side information, when available).
\vspace{-0.1in}

\vspace{-0.05in}
\subsubsection*{Acknowledgements}
\vspace{-0.05in}
We would like to thank the reviewers for their comments which led improving the presentation of this paper.

\subsubsection*{References}

\bibliographystyle{plainnat}

\appendix

\section{Proofs from Section~\ref{s:algbatch}}
\subsection{Proof of Theorem~\ref{t:active}}
The following ancillary results will be useful.
\begin{lemma}[Hoeffding's inequality for sampling without replacement]
\label{l:hoeff-wr}
Let $\scX = \{x_1,\dots,x_N\}$ be a finite subset of $[0,1]$ and let
\[
	\mu = \frac{1}{N}\sum_{i=1}^N x_i~.
\]
If $X_1,\dots,X_n$ is a random sample drawn at random from $\scX$ without replacement, then, for every $\ve > 0$,
\[
	\Pr\left( \left|\frac{1}{n}\sum_{t=1}^n X_t - \mu \right| \ge \ve \right) \le 2e^{-2n\ve^2}~.
\]
\end{lemma}
\begin{lemma}\label{l:bins}
Let $\NN_1, \ldots, \NN_n$ be subsets of a finite set $E$.
Let $E_0 \subseteq E$ be sampled uniformly at random {\em without} replacement from $E$, with $|E_0| = m$. Then, for $\delta \in (0,1)$, $Q > 0$, and $\theta \geq 2\times\max\left\{Q,4\ln \frac{n}{\delta}\right\}$, we have
\[
\Pr\Bigl( \exists i \,:\, |\NN_i| \geq \theta, |\NN_i \cap E_0| < Q\Bigl) \leq \delta
\]
provided $|E| \geq m \geq \frac{2|E|}{\theta}\times\max\left\{Q,4\ln \frac{n}{\delta}\right\}$.
\end{lemma}
\begin{proof}[Proof of Lemma \ref{l:bins}]
Set for brevity $p_i = |\NN_i|/|E|$. Then, due to the sampling without replacement, each random variable $|\NN_i \cap E_0|$ is the sum of $m$ dependent Bernoulli random variables $X_{i,1}, \ldots, X_{i,m}$ such that $\Pr(X_{i,t} = 1) = p_i$, for $t = 1, \ldots, m$. Let $i$ be such that $|\NN_i| \geq \theta$. Then the condition $m \geq \frac{2|E|Q}{\theta}$ implies
\[
Q \leq \frac{m\theta}{2|E|} \leq \frac{m\,p_i}{2} = \frac{\E\big[|\NN_i \cap E_0|\big]}{2}~.
\]
Since the variables $X_{i,j}$ are negatively associated, we may apply a (multiplicative) Chernoff bound~\cite[Section~3.1]{dpbook}.
This gives
\[
\Pr\big(|\NN_i \cap E_0| < Q\big) \leq e^{-\frac{m\,p_i}{8}} \leq e^{-\frac{m\,\theta}{8|E|}}
\]
so that
$
\Pr\bigl( \exists i\,:\, |\NN_i| \geq \theta, |\NN_i \cap E_0| < Q\bigr) \leq n\,e^{-\frac{m\,\theta}{8|E|}}
$,
which is in turn upper bounded by $\delta$ whenever 
$m \geq \frac{8|E|}{\theta}\ln \frac{n}{\delta}$. 
\end{proof}

Let now $E_{\theta} = \theset{(i,j) \in E}{\din(j) \ge \theta,\, \dout(i) \ge \theta} \setminus E_0$, where $E_0 \ss E$ is the set of edges sampled by the learning algorithm of Section~\ref{s:algbatch}. Then Theorem~\ref{t:active} in the main paper is an immediate consequence of the following lemma. 
\begin{lemma}\label{l:active}
Let $G(Y) = (V,E(Y))$ be a directed graph with labels on the edges generated according to the model in Section~\ref{s:gen}.
For all $0 < \alpha,\delta < 1$ and $0 < \ve < \tfrac{1}{16}$, if the learning algorithm of Section~\ref{s:algbatch} is run with parameters $Q = \tfrac{1}{2\ve^2}\ln\frac{4|V|}{\delta}$ and $\alpha$, then with probability at least $1-11\delta$ the predictions $\yhat(i,j)$ satisfy $\yhat(i,j) = y^*(i,j)$ for all $(i,j) \in E_{\theta}$ such that $\big|\eta(i,j) - \tfrac{1}{2}\big| > 8\ve$.
\end{lemma}
\begin{proof}[Proof of Lemma \ref{l:active}]
We apply Lemma~\ref{l:bins} with $\theta = \tfrac{2Q}{\alpha} \ge 2\times\max\big\{Q,4\ln \frac{2|V|+1}{\delta}\big\}$ to the $2|V|+1$ subsets of $E$ consisting of $E_L$ and $\Nin(i),\Nout(i)$, for $i \in V$. We have that, with probability at least $1-\delta$, at least $Q$ edges of $E_L$ are sampled, at least $Q$ edges of $\Nin(i)$ are sampled for each $i$ such that $\NNin(i) \ge \theta$, and at least $Q$ edges of $\Nout(j)$ are sampled for each $j$ such that $\NNout(j) \ge \theta$.
For all $(i,j) \in E_{\theta}$ let
\[
	\pbar_j = \frac{1}{\din(j)}\sum_{i\in\NNin(j)} p_i \quad\text{and}\quad \qbar_i = \frac{1}{\dout(i)}\sum_{j\in\NNout(i)} q_j
\]
and set for brevity $\hdeltain(j) = 1-\hun(j)$ and $\hdeltaout(i) = 1-\htr(i)$.
We now prove that $\hdeltain(j)$ and $\hdeltaout(i)$ are concentrated around their expectations for all $(i,j) \in E_{\theta}$. Consider $\hdeltaout(i)$ (the same argument works for $\hdeltain(j)$). Let $J_1,\dots,J_Q$ be the first $Q$ draws in $E_0 \cap \NNout(i)$ and define
\[
	\muhat_p(i) = \frac{1}{Q} \sum_{t=1}^Q \frac{p_i+q_{J_t}}{2}~.
\]
Applying Lemma~\ref{l:hoeff-wr} to the set $\theset{\tfrac{p_i+q_j}{2}}{j \in \NNout(i)}$, and using our choice of $Q$, we get that $\big|\muhat_p(i) - \mu_p(i)\big| \le \ve$ holds with probability at least $1-\delta/(2|V|)$, where
\[
	\mu_p(i) = \frac{1}{\dout(i)}\sum_{j \in \NNout(i)} \frac{p_i + q_j}{2} = \frac{p_i + \qbar_i}{2}~.
\]
Now consider the random variables $Z_t = \Ind{y_{i,J_t}=1}$, for $t=1,\dots,Q$. Conditioned on $J_1,\dots,J_Q$, these are independent Bernoulli random variables with $\E[Z_t \mid J_t] = \tfrac{p_i+q_{J_t}}{2}$. Hence, applying a standard Hoeffding bound for independent variables and using our choice of $Q$, we get that
\[
	\left|\frac{1}{Q}\sum_{t=1}^Q Z_t - \muhat_p(i) \right| \le \ve 
\]
with probability at least $1-\delta/(2|V|)$ for every realization of $J_1,\dots,J_Q$. 
Since $\hdeltaout(i) = (Z_1+\cdots+Z_Q)/Q$, we get that $\big|\hdeltaout(i) - \mu_p(i)\big| \le 2\ve$ with probability at least $1-2\delta/(2|V|)$. Applying the same argument to $\hdeltain(j)$, and the union bound\footnote
{
The sample spaces for the ingoing and outgoing edges of the vertices occurring as endpoints in $E_{\theta}$ overlap. Hence, in order to prove a uniform concentration result, we need to apply the union bound over the random variables defined over these sample spaces, which motivates the presence of the factor $\ln(2|V|)$ in the definition of $Q$. 
}
on the set $\theset{\hdeltain(j),\hdeltaout(i)}{(i,j) \in E_{\theta}}$, we get that
\begin{equation}
\label{eq:first}
	\left| \hdeltaout(i) + \hdeltain(j) - \frac{p_i + q_j}{2} - \frac{\pbar_j + \qbar_i}{2} \right| \le 4\ve
\end{equation}
simultaneously holds for all $(i,j) \in E_{\theta}$ with probability at least $1 - 4\delta$. Now notice that $\pbar_j$ is a sample mean of $Q$ i.i.d.\ $[0,1]$-valued random variables drawn from the prior marginal $\int_0^1 \mu\big(\cdot,q\bigr) dq$ with expectation $\mu_p$. Similarly, $\qbar_i$ is a sample mean of $Q$ i.i.d.\ $[0,1]$-valued random variables independently drawn from the prior marginal $\int_0^1 \mu\big(p,\cdot\big) dp$ with expectation $\mu_q$. By applying Hoeffding bound for independent variables, together with the union bound to the set of pairs of random variables whose sample means are $\pbar_j$ and $\qbar_i$ for each $(i,j) \in E_{\theta}$ (there are at most $2|V|$ of them) we obtain that
\begin{align*}
	\big|\pbar_j - \mu_p\big| \le \ve
\qquad\text{and}\qquad
	\big|\qbar_i - \mu_q\big| \le \ve
\end{align*}
hold simultaneously for all $(i,j) \in E_{\theta}$ with probability at least $1-2\delta$. Combining with~(\ref{eq:first}) we obtain that
\begin{equation}
\label{eq:this}
	\left| \hdeltaout(i) + \hdeltain(j) - \frac{p_i + q_j}{2} - \frac{\mu_p + \mu_q}{2} \right| \le 5\ve
\end{equation}
simultaneously holds for each $(i,j) \in E_{\theta}$ with probability at least $1 - 6\delta$. Next, let $E_L'$ be the set of the first $Q$ edges drawn in $E_L \cap E_0$. Then
\[
	\E\big[\tauhat\big] = \frac{1}{Q}\sum_{(i,j) \in E_L'} \Pr\big(y_{i,j} = 1\big) = \frac{1}{Q}\sum_{(i,j) \in E_L'} \frac{p_i+q_j}{2}~,
\]
where the expectation is w.r.t.\ the independent draws of the labels $y_{i,j}$ for $(i,j) \in E_L'$. Hence, by applying again Hoeffding bound (this time without the union bound) to the $Q = \tfrac{1}{2\ve^2}\ln\tfrac{2}{\delta}$ independent Bernoulli random variables $\Ind{y_{i,j} = 1}$, $(i,j) \in E_L'$, the event
$
	\big| \tauhat - \E\big[\tauhat\big] \big| \le \ve
$
holds with probability at least $1-\delta$. Now, introduce the function
\[
	F(\bp,\bq) = \E\big[\tauhat\big] = \frac{1}{Q}\sum_{(i,j) \in E_L'} \frac{p_i+q_j}{2}~.
\]
For any realization $\bq_0$ of $\bq$, the function $F_1(\bp) = F(\bp,\bq_0)$ is a sample mean of $Q = \tfrac{1}{2\ve^2}\ln\tfrac{4|V|}{\delta}$ i.i.d.\ $[0,1]$-valued random variables $\theset{p_i}{(i,j) \in E_L'}$ (recall that if $i \in V$ is the origin of an edge $(i,j) \in E_L'$, then it is not the origin of any other edge $(i,j') \in E_L'$). Using again the standard Hoeffding bound, we obtain that
\[
	\left| F(\bp,\bq) - E_{\bp}\big[F(\bp,\bq)\big] \right| \le \ve
\]
holds with probability at least $1 -\delta$ for each $\bq\in [0,1]^{|V|}$. With a similar argument, we obtain that
\[
	\left| E_{\bp}\big[F(\bp,\bq)\big] - E_{\bp,\bq}\big[F(\bp,\bq)\big] \right| \le \ve
\]
also holds with probability at least $1 -\delta$. Since
\[
	E_{\bp,\bq}\big[F(\bp,\bq)\big] = \frac{\mu_p+\mu_q}{2}
\]
we obtain that
\begin{equation}
\label{eq:that}
	\Big| \tauhat - \frac{\mu_p+\mu_q}{2} \Big| \le 3\ve
\end{equation}
with probability at least $1-3\delta$. Combining~(\ref{eq:this}) with~(\ref{eq:that}) we obtain
\[
	\left| \hdeltaout(i) + \hdeltain(j) - \tauhat - \frac{p(i) + q(j)}{2} \right| \le 8\ve
\]
simultaneously holds for each $(i,j) \in E_{\theta}$ with probability at least $1 - 10\delta$. Putting together concludes the proof.
\end{proof}
%

\subsection{Derivation of the maximum likelihood equations}
Recall that the training set
$
	E_0 = \theset{\big(i_t,j_t),y_{i_t,j_t}\big)}{t=1,\dots,m}
$
is drawn uniformly at random from $E$ without replacement. 
We can write
\begin{align*}
	\Pr&\left(E_0\,\Big|\, \{p_i, q_i\}_{i=1}^{|V|}\right)\\ 
&= 
	\frac{1}{\binom{|E|}{m}\,m!}\prod_{k=1}^m \left(\frac{p_{i_k}+q_{j_k}}{2} \right)^{\Ind{y_{i_k,j_k}=+1}}\,\\
&\quad\times\,\prod_{k=1}^m\left(1-\frac{p_{i_k}+q_{j_k}}{2} \right)^{\Ind{y_{i_k,j_k}=-1}}\\
&=
\frac{1}{\binom{|E|}{m}\,m!}
\prod_{\ell=1}^{|V|}\Biggl(\prod_{k=1}^m\,
\left(\frac{p_{\ell}+q_{j_k}}{2} \right)^{\Ind{i_k = \ell,\,y_{\ell,j_k}=+1}}\,\\
&\quad\times\,
\prod_{k=1}^m\left(1-\frac{p_{\ell}+q_{j_k}}{2} \right)^{\Ind{i_k = \ell,\,y_{\ell,j_k}=-1}}\Biggl)
\end{align*}
so that $\log \Pr\left(E_0\,\Big|\, \{p_i, q_i\}_{i=1}^{|V|}\right)$ is proportional to
\begin{align*}
&\sum_{\ell=1}^{|V|}\sum_{k=1}^m \Ind{i_k = \ell,y_{\ell,j_k}=+1}\,
\log\left(\frac{p_{\ell}+q_{j_k}}{2} \right)\,\\
&+\,
\sum_{\ell=1}^{|V|}\sum_{k=1}^m \Ind{i_k = \ell,y_{\ell,j_k}=+1}\, 
\log\left(1-\frac{p_{\ell}+q_{j_k}}{2} \right)
\end{align*}
and
\begin{align*}
\frac{\partial \log \Pr\left(E_0\,\Big|\, \{p_i, q_i\}_{i=1}^{|V|}\right)}{\partial p_{\ell}}
&= 
\sum_{k=1}^m 
\frac{\Ind{i_k = \ell,y_{\ell,j_k}=+1}}{p_{\ell}+q_{j_k}}\,\\
& - \,
\sum_{k=1}^m \frac{\Ind{i_k = \ell,y_{\ell,j_k}=-1}}{2-p_{\ell}-q_{j_k}}\,.
\end{align*}
By a similar argument,
\begin{align*}
\Pr&\left(E_0\,\Big|\, \{p_i, q_i\}_{i=1}^{|V|}\right) 
\\ &=
\frac{1}{\binom{|E|}{m}\,m!}
\prod_{\ell=1}^{|V|}\Biggl(\prod_{k=1}^m\,
\left(\frac{p_{i_k}+q_{\ell}}{2} \right)^{\Ind{j_k = \ell,\,y_{i_k,\ell}=+1}}\,\\
&\quad\times\,
\prod_{k=1}^m\left(1-\frac{p_{i_k}+q_{\ell}}{2} \right)^{\Ind{j_k = \ell,\,y_{i_k,\ell}=-1}}\Biggl)
\end{align*}
so that
\begin{align*}
\frac{\partial \log \Pr\left(E_0\,\Big|\, \{p_i, q_i\}_{i=1}^{|V|}\right)}{\partial q_{\ell}}
&= 
\sum_{k=1}^m 
\frac{\Ind{j_k = \ell,y_{i_k,\ell}=+1}}{p_{i_k}+q_{\ell}}\,\\ 
&- \,
\sum_{k=1}^m 
\frac{\Ind{j_k = \ell,y_{i_k,\ell}=-1}}{2-p_{i_k}-q_{\ell}}\,.
\end{align*}
We then derive the approximation presented in the main paper. Namely, equating to zero the
gradient of the log likelihood w.r.t $p_\ell$ gives
\begin{align*}
\sum_{k=1}^m 
\frac{\Ind{i_k = \ell,y_{\ell,j_k}=+1}}{p_{\ell}+q_{j_k}}\,
 - \,
\frac{\Ind{i_k = \ell,y_{\ell,j_k}=-1}}{2-p_{\ell}-q_{j_k}} &= 0
\end{align*}
To simplify the notation, let $a_k = \Ind{i_k = \ell,y_{\ell,j_k}=+1}$, $b_k=\Ind{i_k =
\ell,y_{\ell,j_k}=+1}$ and $c_k = p_{\ell} + q_{j_k}$, we can rewrite the previous equation as 
\begin{align*}
\sum_{k=1}^m \frac{a_k}{c_k} - \frac{b_k}{2-c_k} &= 0 \\
\sum_{k=1}^m  \frac{a_k(2-c_k) - b_k c_k }{ c_k(2-c_k) } &=0
\end{align*}

The approximation consists in assuming that the denominator $c_k(2-c_k)$ is a
constant for all $k$, and can therefore be disregarded.
Moving $b_kc_k$ to the right hand side and returning to the original variables,
it yields the approximate equation presented in the main paper, namely
\begin{align*}
\sum_{k=1}^m &\Ind{i_k = \ell,y_{\ell,j_k}=+1} \left(2-p_{\ell}-q_{j_k}\right)\\
=
&\sum_{k=1}^m \Ind{i_k = \ell,y_{\ell,j_k}=-1}
(p_{\ell}+q_{j_k})
\end{align*}

\subsection{Label propagation on $G''$}
Here we provide more details on the choice of weight for the edges of $G''$, as
well as an explanation on why we temporarily use symmetrized variables lying in
$[-1, 1]$ (which we will denote with primes, so that for instance $p'_i =
2p_i-1$). Since only the ratio between the negative and positive weights
matters, we fix the negative weight of the edges in $E''\setminus E'$ to be
$-1$ and we denote by $\eps$ the weight of edges in $E'$.  With these notations,
Label Propagation on $G''$ seeks the harmonic minimizer of the following
expression 
\begin{equation*}
  \frac{1}{16} \sum_{i,j \in E} \Bigl[ \eps\left(y_{i,j} - p_i'\right)^2 + \eps\left(y_{i,j} - q_j'\right)^2 + (p'_i+q_j')^2 \Bigr] \\
\end{equation*}
which can be successively rewritten as
\begin{align*}
  \frac{1}{16} \sum_{i,j \in E} & \Bigl[ \eps\left(y_{i,j}+1 - 2p_i\right)^2 + \eps\left(y_{i,j}+1 - 2q_j\right)^2
\\&
	+ (2p_i +2q_j -2)^2 \Bigr]
\\=
	\frac{1}{8} \sum_{i,j \in E} & \Biggl[ 2\eps\!\left(\frac{y_{i,j}+1}{2} - p_i\right)^2\!\! +  2\eps\!\left(\frac{y_{i,j}+1}{2} - q_j\right)^2
\\&
	+ 8\left(\frac{p_i +q_j -1}{2}\right)^2 \Biggr]
\\=
  \frac{1}{8}\sum_{i,j \in E} & \Biggl[ 2\eps\left(\left(\frac{y_{i,j}+1}{2}\right)^2 - p_i(1+y_{i,j}) + p_i^2\right) + \\
                              & 2\eps\left(\left(\frac{y_{i,j}+1}{2}\right)^2 - q_j(1+y_{i,j}) + q_j^2\right) +  \\
                              & 8\left(\left(\frac{p_i +q_j}{2}\right)^2 - \frac{p_i+q_j}{2} + \frac{1}{4} \right) \Biggr]
\\=
  \frac{1}{8}\sum_{i,j \in E} & 4\Biggl(\eps\!\left(\frac{y_{i,j}+1}{2}\right)^2\!\! -
  2\eps\!\left(\frac{y_{i,j}+1}{2}\right)\left(\frac{p_i+q_j}{2}\right)
\\&
  + 2\!\left(\frac{p_i +q_j}{2}\right)^2\Biggr)
\\&+
  \sum_{i,j \in E} \Bigl[ \left(2\eps p_i^2 - 4p_i + 1\right) + \left(2\eps q_j^2 - 4q_j + 1\right) \Bigr]
\end{align*}
By setting $\eps=2$, we can factor this expression into 
\begin{align*}
  \sum_{i,j \in E} &\left(\frac{y_{i,j}+1}{2} - \frac{p_i+q_j}{2}\right)^2
\\&
  + \frac{1}{2}\sum_{i,j \in E} \left(\left(p_i - \frac{1}{2}\right )^2 + \left(q_j - \frac{1}{2}\right)^2 \right)~.
\end{align*}

\section{Proofs from Section~\ref{s:algonline}}

\begin{proof}[Proof of Theorem~\ref{t:online}]
Let each node $i \in V$ host two instances of the randomized Weighted Majority (RWM) algorithm~\cite{LittlestoneWa94} with an online tuning of their learning rate~\cite{cb+97,acg02}: one instance for predicting the sign of outgoing edges $(i,j)$, and one instance for predicting the sign of incoming edges $(j,i)$. Both instances simply compete against the two constant experts, predicting always $+1$ or always $-1$. Denote by $M(i,j)$ the indicator function (zero-one loss) of a mistake on edge $(i,j)$. Then the expected number of mistakes of each RWM instance satisfy~\cite{cb+97,acg02}:
\[
    \sum_{j \in \NNout(i)} \E\,M(i,j) = \Psiout(i,Y) + O\left(\sqrt{\Psiout(i,Y)}+ 1\right)
\]
and
\[
    \sum_{i \in \NNin(j)} \E\,M(i,j) = \Psiin(j,Y) + O\left(\sqrt{\Psiin(j,Y)} + 1\right)~.
\]
We then define two meta-experts: an ingoing expert, which predicts $y_{i,j}$ using the prediction of the ingoing RWM instance for node $j$, and the outgoing expert, which predicts $y_{i,j}$ using the prediction of the outgoing RWM instance for node $i$.
The number of mistakes of these two experts satisfy
\begin{align*}
\sum_{i \in V}\sum_{j \in \NNout(i)} &\E\,M(i,j)\\ 
&= \Psiout(Y) + O\left(\sqrt{|V|\Psiout(Y)} + |V|\right)\\
\sum_{j \in V}\sum_{i \in \NNin(j)} &\E\,M(i,j)\\ 
&= \Psiin(Y)  + O\left(\sqrt{|V|\Psiin(Y)}  + |V|\right)~,
\end{align*}
where we used $\sum_{j \in V} \sqrt{\Psiin(j,Y)} \leq \sqrt{|V|\Psiin(Y)}$, and similarly for $\Psiout(Y)$.
Finally, let the overall prediction of our algorithm be a RWM instance run on top of the ingoing and the outgoing experts. Then the expected number of mistakes of this predictor satisfies
\begin{align*}
    \sum_{(i,j) \in E} \E\,M(i,j) 
&= \Psi_G(Y) + O\Biggl(\sqrt{|V|\Psi_G(Y)} + |V|\\ 
&\,\,+ \sqrt{\left(\Psi_G(Y) + |V| + \sqrt{|V|\Psi_G(Y)}\right)} \Biggl)\\
&= \Psi_G(Y) + O\left(\sqrt{|V|\Psi_G(Y)} + |V|\right)~,
\end{align*}
as claimed.
\end{proof}

\begin{proof}[Proof sketch of Theorem~\ref{t:mistake_bound}]
Let $\mathcal{Y}_K$ be the set of all labelings $Y$ such that the total number of negative and positive edges are $K$ and $|E|-K$, respectively (without loss of generality we will focus on negative edges). Consider the randomized strategy that draws a labeling 
$Y \in \spin^{|E|}$
uniformly at random from $\mathcal{Y}_K$. For each node $i \in V$, we have $\Psiin(i,Y) \le \din^-(i)$, which implies $\Psiin(Y) \le K$. A very similar argument applies to the outgoing edges, leading to $\Psiout(Y) \le K$. The constraint $\Psi_G(Y) \le K$ is therefore always satisfied.

The adversary will force on average $1/2$ mistakes in each one of the first $K$ rounds of the online protocol by repeating $K$ times the following: (i) A label value $\ell\in\{-1,+1\}$ is selected uniformly at random. (ii) An edge $(i,j)$ is sampled uniformly at random from the set of all edges that were not previously revealed and whose labels are equal to $\ell$. 

The learner is required to predict $y_{i,j}$ and, in doing so, $1/2$ mistakes will be clearly made on average because of the randomized labeling procedure. Observe that this holds even when $A$ knows the value of $K$ and $\Psi_G(Y)$. Hence, we can conclude that the expected number of mistakes that $A$ can be forced to make is always at least $K/2$, as claimed.

We now show that, as $\frac{K}{|E|} \rightarrow 0$, the lower bound gets arbitrarily close to $K$ for any $G(Y)$ and any constant $K$.
Let $\mathcal{E}$ be the following event: There is at least one unrevealed negative label.
The randomized iterative strategy used to achieve this result is identical to the one described above, except for the stopping criterion. Instead of repeating step (i) and (ii) only for the first $K$ rounds, these steps are repeated until $\mathcal{E}$ is true.
Let $m_{r,c}$ be defined as follows: For $c=1$ it is equal to the expected number of mistakes forced in round $r$ when $K=1$. For $c > 1$ it is equal to the difference between the expected number of mistakes forced in round $r$ when $K=c$ and $K=c-1$. One can see that $m_{r,c}$ is 
null when $r<c$. 
%
When $K=1$, the probability that $\mathcal{E}$ is true in round $r$ is clearly equal to $\frac{1}{2^{r-1}}$.
Hence, the expected number of mistakes made by $A$ when $K=1$ in any round $r$ is equal to 
\(
\frac{1}{2}\,\frac{1}{2^{r-1}} = \frac{1}{2^{r}}.
\)
We can therefore conclude that $m_{r,1}=\frac{1}{2^{r}}$ for all $r$.

A simple calculation shows that if $r=c$ then $m_{r,c}=\frac{1}{2^r}$. Furthermore, when $r>1$ and $c>1$, 
we have the following recurrence:
\[
m_{r,c}=\frac{m_{r-1,c}+m_{r-1,c-1}}{2}~.
\]  
In order to calculate $m_{r,c}$ for all $r$ and $c$, we will rest on the ancillary quantity $s_j(i)$, recursively defined as specified next.

Given any integer variable $i$, we have
$
s_0(i)=1
$
and, for any positive integer $j$, 
\[
s_j(i)=\sum_{k=1}^i s_{j-1}(k)~.
\]

It is not difficult to verify that 
\[
m_{r,c}=\frac{s_{c-1}(r-c+1)}{2^{r}}. 
\]
Since $s_j(i)=\frac{\langle i\rangle_j}{j!}$, where $\langle i\rangle_j$ is the rising factorial $i(i+1)(i+2)\ldots(i+j-1)$, we have 
\[
m_{r,c}=\frac{\langle r-c+1\rangle_{c-1}}{(c-1)!2^{r}}.
\] 

\bigskip

When $\frac{K}{|E|} \rightarrow 0$, given any integer $K'>1$, the difference between the expected number of mistakes forced when $K=K'$ and $K=K'-1$ is equal to
\begin{align*}
\sum_{r=K'}^{\infty} m_{r,K'}
&=
\frac{1}{(K'-1)!}\sum_{r=K'}^{\infty} \frac{\langle r-K'+1\rangle_{K'-1}}{2^{r}}\\
&=
\frac{1}{(K'-1)!2^{K'-1}}\sum_{r'=1}^{\infty} \frac{\langle r'\rangle_{K'-1}}{2^{r'}}~,
\end{align*}
where we set $r'=r-K'+1$.
Setting $i'=i-1$ and recalling that
\[
\langle i\rangle_j=j!{{i+j-1}\choose{i-1}}~,
\]
we have
\[
\frac{1}{j!}\sum_{i=1}^{\infty}\frac{\langle i\rangle_j}{2^i}=
\sum_{i=1}^{\infty}\frac{{{i+j-1}\choose{i-1}}}{2^i}=
\sum_{i'=0}^{\infty}\frac{{{i'+j}\choose{i'}}}{2^{i'+1}}~.
\]
Now, using the identity
\[
{i'+j+1 \choose i'} = {i'+j \choose i'} + {i'+j \choose i'-1}~,
\]
we can easily prove by induction on $j$ that
\[
\sum_{i'=0}^{\infty}\frac{{{i'+j}\choose{i'}}}{2^{i'+1}}=2^j~.
\]
Hence, we have
\[
\sum_{r=K'}^{\infty} m_{r,K'}=1.
\]
Moreover, as shown earlier, $m_{r,1}=\frac{1}{2^{r}}$ for all $r$. Hence we can conclude that when $\frac{K}{|E|} \rightarrow 0$ 
\[
\E M_A(Y) \ge
\sum_{r=1}^{\infty} \frac{1}{2^{r}} +
\sum_{K'=2}^{K} \sum_{r=K'}^{\infty} m_{r,K'}=K
\]
for any edge-labeled graph $G(Y)$ and any constant $K$, as claimed.
\end{proof}

\section{Further Experimental Results}
This section contains more evidence related to the experiments in Section~\ref{s:exp}.
In particular, we experimentally demonstrate the alignment between \usrule{} and \uslogregp{}.

After training on the two features $1-\htr(i)$ and $1-\hun(j)$, \uslogregp{} has learned three weights
$w_0$, $w_1$ and $w_2$, which allow to predict $y_{i,j}$ according to 
\[ 
  \sgn\Big(\big(w_1(1-\htr(i)\big) + w_2 \big(1-\hun(j)\big) + w_0\Big)~.
\]
This can be rewritten as 
\[
\sgn\Big(\big(1-\htr(i)\big) + w_2'\big(1-\hun(j)\big) - \tfrac{1}{2} -\tau'\Big)~,
\]
with $w_2' = \frac{w_2}{w_1}$ and $\tau' = - \left(\frac{1}{2} + \frac{w_0}{w_1}\right)$~.
	
As shown in Table~\ref{tab:coeff}, and in accordance with the predictor built out of Equation~\eqref{eq:predictor},
$w_2'$ is almost $1$ on {\em all datasets}, while $\tau'$ tends to be always close the fraction of positive edges in the dataset.

\begin{table}[bt]
  \centering
  \caption{Normalized logistic regression coefficients averaged over 12 runs (with one standard deviation) \label{tab:coeff}}
  \small
\begin{tabular}{lrccc}
\toprule
                        & $\frac{|\trainset{}|}{|E|}$ &              $w_2'$ & $\tau'$ \\
\multirow{5}{*}{\aut{}} & $5\%$  &  $0.965 \pm 0.04$  &  $0.662 \pm 0.03$  \\
                        & $10\%$ &  $0.983 \pm 0.03$  &  $0.705 \pm 0.02$  \\
                        & $15\%$ &  $1.001 \pm 0.03$  &  $0.729 \pm 0.03$  \\
                        & $20\%$ &  $1.013 \pm 0.02$  &  $0.747 \pm 0.02$  \\
                        & $25\%$ &  $1.011 \pm 0.02$  &  $0.746 \pm 0.01$  \\
\midrule
\multirow{5}{*}{\wik{}} & $5\%$  &  $0.920 \pm 0.02$  &  $0.691 \pm 0.02$  \\
                        & $10\%$ &  $0.940 \pm 0.01$  &  $0.730 \pm 0.01$  \\
                        & $15\%$ &  $0.947 \pm 0.01$  &  $0.741 \pm 0.01$  \\
                        & $20\%$ &  $0.963 \pm 0.01$  &  $0.760 \pm 0.01$  \\
                        & $25\%$ &  $0.962 \pm 0.02$  &  $0.764 \pm 0.01$  \\
\midrule
\multirow{5}{*}{\sla{}} & $5\%$  &  $1.024 \pm 0.02$  &  $0.693 \pm 0.01$  \\
                        & $10\%$ &  $1.017 \pm 0.01$  &  $0.705 \pm 0.01$  \\
                        & $15\%$ &  $1.007 \pm 0.01$  &  $0.707 \pm 0.01$  \\
                        & $20\%$ &  $1.002 \pm 0.00$  &  $0.710 \pm 0.00$  \\
                        & $25\%$ &  $0.995 \pm 0.01$  &  $0.712 \pm 0.00$  \\
\midrule
\multirow{5}{*}{\epi{}} & $5\%$  &  $1.099 \pm 0.02$  &  $0.791 \pm 0.02$  \\
                        & $10\%$ &  $1.059 \pm 0.01$  &  $0.782 \pm 0.01$  \\
                        & $15\%$ &  $1.037 \pm 0.01$  &  $0.774 \pm 0.01$  \\
                        & $20\%$ &  $1.018 \pm 0.01$  &  $0.765 \pm 0.01$  \\
                        & $25\%$ &  $1.010 \pm 0.01$  &  $0.763 \pm 0.01$  \\
\midrule
\multirow{5}{*}{\kiw{}} & $5\%$  &  $1.047 \pm 0.02$  &  $0.853 \pm 0.01$  \\
                        & $10\%$ &  $1.038 \pm 0.01$  &  $0.872 \pm 0.01$  \\
                        & $15\%$ &  $1.025 \pm 0.01$  &  $0.876 \pm 0.01$  \\
                        & $20\%$ &  $1.012 \pm 0.01$  &  $0.874 \pm 0.01$  \\
                        & $25\%$ &  $1.007 \pm 0.01$  &  $0.874 \pm 0.01$  \\
\bottomrule
  \end{tabular}
\end{table}

\subsubsection*{References}

\bibliographystyle{plainnat}

\begin{thebibliography}{11}


\vspace{-0.0in}
\bibitem{mcc00}
P. Baldi, S. Brunak, Y. Chauvin, C.A. Andersen, and H. Nielsen.
Assessing the accuracy of prediction algorithms for classification: an overview.
Bioinformatics, 5, 16, pp. 412--424, 2000.

\vspace{-0.0in}
\bibitem{BDL06}
Y.~Bengio, O.~Delalleau, and N.~{Le Roux}.
\newblock Label propagation and quadratic criterion.
\newblock In \emph{Semi-Supervised Learning}, 193--216. {MIT} Press, 2006.

\vspace{-0.0in}
\bibitem{BC01}
A.~Blum and S.~Chawla.
\newblock Learning from labeled and unlabeled data using graph mincuts.
\newblock In \emph{18th ICML}, pages 19--26. Morgan Kaufmann, 2001.

\vspace{-0.0in}
\bibitem{cartwright1956structural}
D. Cartwright and F. Harary.
\newblock Structural balance: a generalization of Heider's theory.
\newblock \emph{Psychological review}, 63\penalty0 (5):\penalty0 277, 1956.



\vspace{-0.0in}
\bibitem{CCCC12}
N. Cesa-Bianchi, C. Gentile, F. Vitale, and G. Zappella.
\newblock {A correlation clustering approach to link classification in signed
  networks}.
\newblock In \emph{25th COLT},
  2012{\natexlab{a}}.

\vspace{-0.0in}
\bibitem{TreeStar12}
N. Cesa-Bianchi, C. Gentile, F. Vitale, and G. Zappella.
\newblock {A linear time active learning algorithm for link classification}.
\newblock In \emph{NIPS 25}, 2012{\natexlab{b}}.

\vspace{-0.0in}
\bibitem{cgvz13}
N.~Cesa-Bianchi, C.~Gentile, F.~Vitale, G.~Zappella.
Random spanning trees and the prediction of weighted graphs.
{\em JMLR}, 14, pp. 1251--1284.

\vspace{-0.0in}
\bibitem{TrollDetection15}
J. Cheng, C. Danescu-Niculescu-Mizil, and J. Leskovec.
\newblock Antisocial behavior in online discussion communities.
\newblock In \emph{Intl AAAI Conf. on Web and Social Media}, 2015.

\vspace{-0.0in}
\bibitem{LowRankCompletion14}
K. Chiang, C. Hsieh, N. Natarajan, I. Dhillon, and A. Tewari.
\newblock {Prediction and Clustering in Signed Networks: A Local to Global
  Perspective}.
\newblock \emph{JMLR}, 15:\penalty0 1177--1213,
  2014.

\vspace{-0.0in}
\bibitem{davis1967clustering}
J.A. Davis.
\newblock Clustering and structural balance in graphs.
\newblock \emph{Human relations}, 1967.


\vspace{-0.0in}
\bibitem{guha2004propagation}
R. Guha, R. Kumar, P. Raghavan, and A. Tomkins.
\newblock Propagation of trust and distrust.
\newblock In \emph{13th WWW}, pp. 403--412, 2004.


\vspace{-0.0in}
\bibitem{heider1958psychology}
F. Heider.
\newblock \emph{The psychology of interpersonal relations}.
\newblock 1958.



\vspace{-0.0in}
\bibitem{HP07}
M.~Herbster and M.~Pontil.
\newblock Prediction on a graph with the {P}erceptron.
\newblock In \emph{NIPS 21}, pp. 577--584. MIT Press, 2007.

\vspace{-0.0in}
\bibitem{HLP09}
M.~Herbster, G.~Lever, and M.~Pontil.
\newblock Online prediction on large diameter graphs.
\newblock In \emph{NIPS 22}, pp. 649--656.
MIT Press, 2009{\natexlab{a}}.




\bibitem{hl81}
P. W. Holland and S. Leinhardt. An Exponential Family of Probability Distributions for Directed Graphs, {\em JASA}, 76, pp. 33--65, 1981.

\vspace{-0.0in}
\bibitem{de2008pagetrust}
C. De~Kerchove and P. Van~Dooren.
\newblock The pagetrust algorithm: How to rank web pages when negative links are allowed?
\newblock In {\em SDM}, pp. 346--352. SIAM, 2008.

\vspace{-0.0in}
\bibitem{kumar2016structure}
S. Kumar. Structure and Dynamics of Signed Citation Networks. In \emph{25th WWW}, 2016.

\vspace{-0.0in}
\bibitem{Kunegis2009}
J. Kunegis, A. Lommatzsch, and C. Bauckhage.
\newblock {The Slashdot Zoo: Mining a Social Network with Negative Edges}.
\newblock In \emph{18th WWW}, pp. 741, 2009.

\vspace{-0.0in}
\bibitem{Leskovec2010}
J. Leskovec, D. Huttenlocher, and J. Kleinberg.
\newblock {Predicting positive and negative links in online social networks}.
\newblock In \emph{19th WWW}, pp. 641--650, 2010.

\vspace{-0.0in}
\bibitem{litt88}
N. Littlestone.
\newblock Learning quickly when irrelevant attributes abound: A new linear-threshold algorithm.
\newblock {\em Machine Learning}, 2(4):285--318, 1988.


\vspace{-0.0in}
\bibitem{wikiedits11}
S. Maniu, T. Abdessalem, B. and Cautis.
Casting a Web of Trust over Wikipedia: An Interaction-based Approach.
In \emph{20th WWW}, pp. 87--88, 2011.

\vspace{-0.0in}
\bibitem{mishra2011finding}
A. Mishra and A. Bhattacharya.
\newblock Finding the bias and prestige of nodes in networks based on trust scores.
\newblock In \emph{20th WWW}, pp. 567--576. ACM, 2011.

\vspace{-0.0in}
\bibitem{Papaoikonomou2014}
A. Papaoikonomou, M. Kardara, K. Tserpes, and T.A. Varvarigou.
\newblock {Predicting Edge Signs in Social Networks Using Frequent Subgraph
  Discovery}.
\newblock \emph{IEEE Internet Computing}, 18\penalty0 (5):\penalty0 36--43,
  2014.

\vspace{-0.0in}
\bibitem{EdgeSignsRating15}
A. Papaoikonomou, M. Kardara, and T.A. Varvarigou.
\newblock {Trust Inference in Online Social Networks}.
\newblock In \emph{Proc. Intl. Conf. on Advances in Social Networks Analysis and Mining}, pp. 600--604, 2015.

\vspace{-0.0in}
\bibitem{Qian2014sn}
Y. Qian and S. and Adali.
\newblock {Foundations of Trust and Distrust in Networks: Extended Structural
  Balance Theory}.
\newblock \emph{ACM Trans. Web}, 8\penalty0 (3):\penalty0 13:1--13:33, 2014.


\vspace{-0.0in}
\bibitem{shahriari2014ranking}
M. Shahriari and M. Jalili.
\newblock Ranking nodes in signed social networks.
\newblock {\em Social Network Analysis and Mining}, 4(1):1--12, 2014.

\vspace{-0.0in}
\bibitem{Bayesian15}
D. Song and D.A. Meyer.
\newblock Link sign prediction and ranking in signed directed social networks.
\newblock {\em Social Network Analysis and Mining}, 5(1):1--14, 2015.

\vspace{-0.0in}
\bibitem{tang2013exploiting}
J. Tang, H. Gao, X. Hu, and H. Liu.
\newblock Exploiting homophily effect for trust prediction.
\newblock In \emph{6th WSDM}, pp. 53--62, 2013.

\vspace{-0.0in}
\bibitem{tang2015survey}
J. Tang, Y. Chang, C. Aggarwal, H. Liu.
\newblock A Survey of Signed Network Mining in Social Media.
\newblock \emph{arXiv preprint arXiv:1511.07569}, 2015

\vspace{-0.0in}
\bibitem{traag2010exponential}
V.A. Traag, Y.E. Nesterov, and P. Van~Dooren.
\newblock {\em Exponential Ranking: taking into account negative links}.
\newblock Springer, 2010.


\vspace{-0.0in}
\bibitem{wu2016troll}
Z. Wu, C. Aggarwal, and J. Sun.
\newblock The troll-trust model for ranking in signed networks.
\newblock In {\em 9th WSDM}, pages 447--456. ACM, 2016.

\vspace{-0.0in}
\bibitem{SNTransfer13}
J. Ye, H. Cheng, Z. Zhu, and M. Chen.
\newblock {Predicting Positive and Negative Links in Signed Social Networks by Transfer Learning}.
\newblock In \emph{22nd WWW}, pp. 1477--1488, 2013.

\vspace{-0.0in}
\bibitem[Zheng \& Skillicorn(2015)Zheng and Skillicorn]{SignedEmbedding15}
Zheng, Q and Skillicorn, D.B.
\newblock \emph{{Spectral Embedding of Signed Networks}}, chapter~7, pp.\
  55--63.
\newblock 2015.

\vspace{-0.0in}
\bibitem{ZGL03}
X.~Zhu, Z.~Ghahramani, and J.~Lafferty.
\newblock Semi-supervised learning using {G}aussian fields and harmonic
  functions.
\newblock In \emph{ICML Workshop on the Continuum from Labeled to Unlabeled
  Data in Machine Learning and Data Mining}, 2003.

\vspace{-0.0in}
\bibitem{zolfaghar2010mining}
K. Zolfaghar and A. Aghaie.
\newblock Mining trust and distrust relationships in social web applications.
\newblock In {\em IEEE ICCP}, pp. 73--80. IEEE, 2010.

\end{thebibliography}

\begin{thebibliography}{1}


\bibitem{acg02}
P. Auer, N. Cesa-Bianchi, and C. Gentile.
\newblock Adaptive and self-confident on-line learning algorithms.
\newblock {\em J. Comput. Syst. Sci.}, 64(1):48--75, 2002.

\bibitem{cb+97}
N.~Cesa-Bianchi, Y.~Freund, D.~Haussler, D.~P. Helmbold, R.~E. Schapire, and M.~K. Warmuth.
\newblock How to use expert advice.
\newblock {\em J. ACM}, 44(3):427--485, 1997.

\bibitem{dpbook}
D. P. Dubhashi, A. Panconesi. Concentration of Measure for the Analysis of Randomized Algorithms
\newblock Cambridge University Press, 2009.

\bibitem{LittlestoneWa94}
N. Littlestone and M. K. Warmuth.
\newblock The weighted majority algorithm.
\newblock {\em Information and Computation}, 108:212--261, 1994.

\end{thebibliography}

\end{document}